\documentclass{ecai} 

\usepackage{latexsym}
\usepackage{amssymb}
\usepackage{amsmath}
\usepackage{amsthm}
\usepackage{booktabs}
\usepackage{enumitem}
\usepackage{graphicx}
\usepackage{color}

\usepackage{mathtools}
\usepackage{hyperref}

\usepackage[capitalize,noabbrev]{cleveref}
\usepackage{microtype}

\usepackage{algorithm}
\usepackage{algorithmic}

\theoremstyle{plain}
\newtheorem{theorem}{Theorem}[section]
\newtheorem{proposition}[theorem]{Proposition}
\newtheorem{lemma}[theorem]{Lemma}
\newtheorem{corollary}[theorem]{Corollary}
\theoremstyle{definition}

\theoremstyle{remark}

\newcommand{\BibTeX}{B\kern-.05em{\sc i\kern-.025em b}\kern-.08em\TeX}

\usepackage{enumitem}
\usepackage{amssymb}
\usepackage{dsfont}

\newcommand\un{{\mathfrak{n}}}
\newcommand\um{{\mathfrak{m}}}

\newcommand\bE{\mathbb{E}}
\newcommand\bP{\mathbb{P}}
\newcommand\bR{\mathbb{R}}
\newcommand\bN{\mathbb{N}}
\newcommand\bZ{\mathbb{Z}}

\newcommand\cP{\mathcal{P}}
\newcommand\cX{\mathcal{X}}

\newcommand\re{\mathrm{e}}
\newcommand\rNN{\mathrm{NN}}

\newcommand\one[1]{\mathbf{1}\mathopen{}\left[#1\right]\mathclose{}}

\newcommand\param{w}

\DeclareMathOperator{\cov}{Cov}

\usepackage[acronym]{glossaries}
\glsdisablehyper
\newacronym{CS}{CS}{central server}
\newacronym{FIFO}{FIFO}{first-in-first-out}
\newacronym{FL}{FL}{federated learning}
\newacronym{iid}{iid}{independent and identically distributed}
\newacronym{SGA}{SGA}{stochastic gradient ascent}
\newacronym{SGD}{SGD}{stochastic gradient descent}

\allowdisplaybreaks  

\begin{document}


\begin{frontmatter}


\paperid{6472} 


\title{Optimizing Asynchronous Federated Learning: A Delicate Trade-Off Between Model-Parameter Staleness and Update Frequency}


\author[A,B]{\fnms{Abdelkrim}~\snm{Alahyane}\thanks{Corresponding Author. Email: abdelkrim.alahyane@um6p.ma.}}
\author[B]{\fnms{Céline}~\snm{Comte}}
\author[B]{\fnms{Matthieu}~\snm{Jonckheere}} 
\author[C]{\fnms{Éric}~\snm{Moulines}} 

\address[A]{EMINES, Mohammed VI Polytechnic University, Ben Guerir, Morocco}
\address[B]{LAAS-CNRS, Université de Toulouse, CNRS, Toulouse, France}
\address[C]{Ecole Polytechnique, Palaiseau, France}


\begin{abstract}
    Synchronous federated learning (FL) scales poorly with the number of clients due to the straggler effect.
    Algorithms like \texttt{FedAsync} and \texttt{GeneralizedFedAsync} address this limitation by enabling asynchronous communication between clients and the central server. In this work, we rely on stochastic modeling and analysis to better understand the impact of design choices in asynchronous FL algorithms, such as the concurrency level and routing probabilities, and we leverage this knowledge to optimize loss. Compared to most existing studies, we account for the joint impact of heterogeneous and variable service speeds and heterogeneous datasets at the clients. We characterize in particular a fundamental trade-off for optimizing asynchronous FL: minimizing gradient estimation errors by avoiding model parameter staleness, while also speeding up the system by increasing the throughput of model updates. Our two main contributions can be summarized as follows. First, we prove a discrete variant of Little's law to derive a closed-form expression for relative delay, a metric that quantifies staleness. This allows us to efficiently minimize the average loss per model update, which has been the gold standard in literature to date, using the upper-bound of Leconte et al.\ as a proxy. Second, we observe that naively optimizing this metric drastically slows down the system by overemphasizing staleness at the expense of throughput. This motivates us to introduce an alternative metric that also accounts for speed,
    for which we derive a tractable upper-bound that can be minimized numerically. Extensive numerical results show these optimizations enhance accuracy by 10\% to 30\%.
\end{abstract}

\end{frontmatter}


\section{INTRODUCTION}

\Gls{FL} is a learning paradigm that enables distributed model training across multiple clients under the supervision of a \gls{CS}, without requiring data sharing \cite{konevcny2015federated, mcmahan2017communication}. In synchronous \gls{FL}, \gls{SGD} is executed in rounds; during each round, the \gls{CS} sends the current model parameters to a subset of clients, waits until all of them return a new stochastic gradient estimate, and then updates the model parameters before sending them to another batch of clients. Unfortunately, the performance of synchronous \gls{FL} is often hindered by the variability of computational speeds among clients, leading to the straggler effect \cite{chen2021towards}.

Asynchronous algorithms such as \texttt{FedAsync} \cite{xie2019asynchronous, chen2020asynchronous, xu2021asynchronous}, \texttt{FedBuff} \cite{Nguyen2021FederatedLW}, \texttt{AsGrad} \cite{islamov2023asgrad}, and \texttt{AsyncSGD} \cite[Algorithm~2]{koloskova2022sharper} aim to tackle these challenges by allowing clients and the \gls{CS} to communicate asynchronously. In particular, the \gls{CS} may update the model parameters while clients estimate gradients (possibly based on outdated model parameters). Asynchrony is implemented by allowing tasks (i.e., requests for gradient estimates) to be queued at the clients \cite{koloskova2022sharper}.
Intuitively, asynchronous \gls{FL} has the potential to speed up the system by circumventing the straggler effect, possibly at the cost of errors due to the staleness of the model parameters used to estimate gradients.

Consequently, a significant research effort has been devoted to analyze the performance of asynchronous \gls{FL} in both homogeneous and heterogeneous data settings \cite{chen2020asynchronous, chen2021towards, xu2021asynchronous,koloskova2022sharper}. However, as we will review in \Cref{sec:related-work},
most existing studies overlook the critical impact of queueing dynamics on the actual performance of asynchronous \gls{FL}.
In contrast, \cite{leconte2024queueing} recently showed that the performance of asynchronous \gls{FL} depends critically on these queueing dynamics via the staleness of the model parameters used by clients to estimate gradients. Staleness is captured by the \emph{relative delay}, defined as the (stochastic) number of times the \gls{CS} updates the model parameters while a gradient-estimation task is being held at a client (either queued or being processed). While existing analyses assume this relative delay is bounded irrespective of the system parameters, the results of \cite{leconte2024queueing} imply that the relative delay depends heavily on these parameters and may grow arbitrarily large even in realistic scenarios. This stands in sharp contrast to existing studies such as \cite{tyurin2023optimal,tyurin_shadowheart_2024,maranjyan_mindflayer_2024,maranjyan_ringmaster_2025}.

In this paper,
we derive fundamental insights and actionable tools
for optimizing performance in asynchronous \gls{FL}.
Unlike most works from the literature, our results apply to systems where clients are \emph{heterogeneous}, both in terms of computation speeds and of datasets.
First, building on \cite{leconte2024queueing},
we derive an explicit expression
for the mean relative delay and its gradient
by leveraging the framework of Jackson networks~\cite{jackson1957networks},
which in turn allows us to design a gradient-descent algorithm
that optimizes performance.

Next, we observe that minimizing the average gradient norm per update, as commonly done in the literature, can be counterproductive in asynchronous \gls{FL}, as it ignores update throughput and slows the system to the pace of the slowest client, thereby underutilizing the others.
Although we focus on an extension of \texttt{AsyncSGD} called \texttt{Generalized AsyncSGD}, we believe our results are relevant to other asynchronous \gls{FL} algorithms.

\subsection{Related work} \label{sec:related-work}

Synchronous \gls{FL} has major pitfalls~\cite{wang2020tackling, qu2021feddq, makarenko2022adaptive, mao2022communication, tyurin2022dasha}: the straggler effect leads to important delays, and synchronization becomes challenging when the number of clients increases~\cite{xie2019asynchronous}. These limitations spurred the development of asynchronous \gls{FL} algorithms, such as \texttt{FedAsync}~\cite{xie2019asynchronous, chen2020asynchronous, xu2021asynchronous}, \texttt{FedBuff}~\cite{Nguyen2021FederatedLW}, \texttt{AsGrad}~\cite{islamov2023asgrad}, and \texttt{AsyncSGD}~\cite[Algorithm~2]{koloskova2022sharper}, which incorporate memory-based updates, adaptive learning rate adjustments, and strategies to reduce staleness and handle varying computation speeds.
\cite{mishchenko2023, koloskova2022sharper, cohen_asynchronous_2021} showed in a simplified model that Asynchrounous \gls{SGD} is provably faster in terms of wall-clock time than Minibatch \gls{SGD}.

Despite these advances and the plethora of asynchronous \gls{FL} algorithms that have been proposed, there is still little understanding of the impact of system design on performance.
In real-world distributed learning, compute times are unpredictable and heterogeneous due to hardware failures, preemptions, GPU delays, and network issues, making fixed-time assumptions unrealistic \cite{dutta2018slow, chen2017revisitingdistributedsynchronoussgd, kairouz2021advances}. Instead, compute times should be treated as dynamic client-dependent random variables.
Furthermore, in many applications, datasets are heterogeneous across clients.
In contract, as we see now, most existing analyses fail to provide solutions that work well when clients are heterogeneous both in terms of compute speeds and datasets.

Several works have attempted to account for client heterogeneity in asynchronous \gls{FL}.
The analysis of the celebrated \texttt{FedBuff} algorithm~\cite{Nguyen2021FederatedLW} allows for heterogeneous datasets, but it assumes that the next client completing a gradient computation is sampled uniformly at random, which is especially unrealistic when service speeds are heterogeneous.
\cite{agarwal_fedecado_2024} designs an asynchronous \gls{FL} algorithm robust to dataset heterogeneity and variability in compute times, but it assumes that compute times are sampled from the same distribution for all clients, so that it does not accommodate heterogeneity of client speeds. \cite{tyurin2023optimal,tyurin_shadowheart_2024,maranjyan_mindflayer_2024,maranjyan_ringmaster_2025,cohen_asynchronous_2021} attempt to model client speed heterogeneity by allowing for fixed but heterogeneous delays. However, their approach enforces synchronous updates by setting a time threshold and discarding clients that exceed it. This biases the global model against underrepresented data and wastes near-complete computations, reducing overall system efficiency. Furthermore, by assuming that delay is fixed, these works implicitly assume that the system parameters have a bounded impact on the delay, which is no longer true when queueing dynamics are taken into account.
\cite{koloskova2022sharper} explicitly accounts for dataset heterogeneity, but it again makes the unrealistic assumption that delay is bounded irrespective of the system parameters.

This discussion reveals a major gap in the literature: existing analyses fail to fully capture the interplay between, on the one hand, the unpredictable and heterogeneous nature of client speeds and network conditions, and, on the other hand, heterogeneous datasets. This is all the more critical since these properties motivated the introduction of asynchronous \gls{FL} in the first place. 

A preliminary attempt is made in~\cite{leconte2024queueing}, which introduces \texttt{Generalized AsyncSGD}, an algorithm that utilizes non-uniform client selection to address queueing dynamics, heterogeneous client speeds, and heterogeneous datasets. However, their analysis falls short of providing explicit performance bounds, instead relying on scaling regimes to approximate the system behavior.

\subsection{Contributions} \label{sec:contributions}

Our findings stem from key theoretical insights that allow an explicit characterization of the impact of queueing dynamics on the performance of asynchronous \gls{FL}. They can be summarized as follows:

\textbf{Derive tractable bounds on loss gradients.}
Using the framework of queueing theory,
we derive exact and tractable formulas for the mean relative delay and its gradient.
In general, these formulas can be estimated in time $\mathcal{O}(n^2 m^2)$,
where $n$ is the number of clients
and $m$ the concurrency level,
or they can be estimated through Monte Carlo simulations.
Further simplifications allow us to compute them in time $\mathcal{O}(n)$ for simple routing strategies.

\textbf{Optimize performance.}
Leveraging this result, we design an algorithm that improves the performance of \texttt{Generalized AsyncSGD} by using the bound from~\cite{leconte2024queueing} as a proxy objective for routing optimization.
Our result also allows us to gauge the bound's sensitivity to crucial system parameters,
such as the ratio of the concurrency level to the number of clients,
and the clients' service speeds.
These insights are relevant to other asynchronous \gls{FL} algorithms and underscore that routing strategies should be adapted based on application-specific bottlenecks.

\textbf{Account for clock-time performance.}
We observe both analytically and numerically
that minimizing the average norm-square of the gradient per update
actually leads us to slow down the system considerably
by underutilizing all clients but (the slowest) one.
Roughly speaking, since this metric ignores the throughput of model updates,
it is optimized by minimizing staleness,
which is achieved by 
giving priority to the slowest client.
This motivates us to introduce an alternative metric that explicitly accounts for throughput, and for which we derive an upper bound that can again be optimized tractably.

Our findings provide not only qualitative insights into the impact of queueing dynamics, but also efficient numerical methods to optimize performance.
We show in particular that routing strategies should be adapted based on the specific bottlenecks imposed by applications, such as the number of computation rounds versus actual training time.
Our experiments on real-world datasets show that tuning the routing strategy and/or the concurrency level can improve accuracy by 10\% to 30\%.

\subsection{Notations}
$\bZ, \bN, \bN_{> 0}, \bR, \bR_{\ge 0}, \bR_{> 0}$ denote the sets of integers, non-negative integers, positive integers, real numbers, non-negative real numbers, and positive real numbers.
Let $| \cdot |$ denote the $\ell_1$-norm and $\one{\,\cdot\,}$ the indicator function.
For each $n, m \in \bN_{> 0}$,
let $\cX_{n, m} = \{x \in \bN^n: |x| = m\}$
denote the set of $n$-dimensional natural-number-valued vectors with $\ell_1$-norm~$m$.
For every $n \in \bN_{> 0}$,
let $\cP_n = \{p \in \bR^n: 0 < p_i < 1 \text{ for } i \in \{1, 2, \ldots, n\} \text{ and } |p| = 1\}$.

\glsresetall

\section{MODEL AND PRIOR RESULTS}

\subsection[Asynchronous federated learning]{Asynchronous \acrlong{FL}} \label{sec:fl}

The goal in \gls{FL} is to optimize the average performance of a model across multiple clients under the supervision of a \gls{CS}: $\min_{\param \in \bR^d} f(w)$,
where $f(\param) = \frac{1}{n} \sum_{i=1}^n f_i(\param)$, and
\begin{align*}
	f_i(\param) &= \mathbb{E}_{(x, y) \sim \mathcal{D}_i}[\ell_{i}(\mathrm{NN}(x, \param), y)],
	\quad i \in \{1, 2, \ldots, n\}.
\end{align*}
Here, $\param$ denotes the parameters of a deep neural network, $d$ the number of parameters (including weights and biases), $\rNN(x, \param)$ the prediction function of the neural network, $n$ the number of clients, $\ell_{i}$ the local loss function of client~$i$, and $\mathcal{D}_i$ the data distribution at client~$i$. Each client~$i$ approximates the gradient $\nabla_w f_i(w)$ of its local loss function using a stochastic gradient denoted by~$g_{i}(w)$. The computation of such a stochastic gradient by a client is called a \emph{task}. 

\texttt{Generalized AsyncSGD} \cite{leconte2024queueing}
is shown in Algorithms~\ref{alg:CS} (\gls{CS})
and~\ref{alg:client} (client~$i$).
A task assigned to a busy client is queued according to the \gls{FIFO} policy.
As we will see, performance
depends critically on two parameters:
$p$, the routing probability vector of tasks to clients;
and $m$, the concurrency~\cite{koloskova2022sharper},
defined as the number of tasks that are concurrently dispatched
to the clients (either queued or being processed).
\texttt{Generalized AsyncSGD} simplifies to the classical \texttt{AsyncSGD} algorithm~\cite[Algorithm~2]{koloskova2022sharper} when $p = p^{\text{uniform}}$ with $p^{\text{uniform}}_i = \frac1n$ for each $i \in \{1, 2, \ldots, n\}$.

Let us focus on \Cref{alg:CS} (\gls{CS} perspective).
The model parameter is initialized to a random vector $\param_0$,
and the system state is initialized to a random vector
$\xi = (\xi_1, \xi_2, \ldots, \xi_n) \in \cX_{n, m}$,
where $\xi_i$ is the number of tasks initially dispatched to client~$i$,
for each $i \in \{1, 2, \ldots, n\}$.
After this initialization, each iteration~$t \in \{0, 1, \ldots, T\}$
of the \texttt{for} loop (Line~7) proceeds as follows:
whenever a client~$C_t$
completes a task and reports a gradient estimate $g_{C_t}(w_{I_t})$
(Line~8; $I_t$ will be defined shortly),
the \gls{CS} immediately updates the model parameter
(Line~9)
and sends it to the next client $A_{t+1}$,
where $\bP(A_{t+1} = i) = p_i$
for each $i \in \{1, 2, \ldots, n\}$
(Lines~10 and~11).
Recall that a client can be chosen
even if it is processing a task.
Observe that the step size in the model-parameter update step is divided by the routing probability to avoid biasing the model towards clients that are sampled more often.

\begin{algorithm}[tbp]
   \caption{\texttt{Generalized AsyncSGD (CS)}}
   \label{alg:CS}
   \begin{algorithmic}[1]
      \STATE {\bf Input:} Numbers $T$, $n$, and $m$ of rounds, clients, and tasks; routing $p$; learning rate $\eta$
      \STATE Initialize parameters $\param_0$ randomly
      \STATE Initialize state vector $\xi \in \cX_{n, m}$ randomly
      \FOR{$i = 1, 2, \ldots, n$}
         \STATE Send $\xi_i$ times model parameter $\param_0$ to client $i$
      \ENDFOR
      \FOR{$t = 0, \dots, T$ \label{line:for}}
         \STATE \gls{CS} receives stochastic gradient $g_{C_t}(\param_{I_t})$ from a client $C_t$ \label{line:CS1}
         \STATE Update $\param_{t+1} \leftarrow \param_t - \frac{\eta}{n p_{C_t}} g_{C_t}(\param_{I_t})$ \label{line:CS2}
         \STATE Sample a new client $A_{t+1}$ according to $p$ \label{line:CS3}
         \STATE Send model parameter $\param_{t+1}$ to client $A_{t+1}$
      \ENDFOR
   \end{algorithmic}
\end{algorithm}

\begin{algorithm}[tbp]
\caption{\texttt{Generalized~AsyncSGD~(Client~$i$)}}
   \label{alg:client}
   \begin{algorithmic}[1]
      \STATE {\bf Input:} Queue of received parameters, local dataset
      \IF{Queue is not empty}
         \STATE Take the received parameter $\param$ from the queue using a FIFO policy
         \STATE Compute the gradient estimate $g_{i}(\param)$
         \STATE Send the gradient to the \gls{CS}
         \STATE Repeat
      \ENDIF
   \end{algorithmic}
\end{algorithm}

Time indices are such that,
for each $t \in \{1, 2, \ldots, T\}$,
$A_t$ and $C_t$ correspond to
the same (model-parameter update) round, 
where a round is defined as the time between the assignment of a task to a client and the next task completion (leading to a model update). In \Cref{sec:optimize-time}, we will see that throughput, defined as the inverse of the (clock-time) duration of a typical round, has a crucial impact on performance.
For each $t \in \{0, 1, \ldots, T\}$, we let $X_t = (X_{1, t}, X_{2, t}, \ldots, X_{n, t}) \in \cX_{n, m-1}$ denote the state at the end of round~$t$, so that, for each $i \in \{1, 2, \ldots, n\}$, we have $X_{i, 0} = \xi_i - \one{C_0 = i}$ and, for each $t \in \{1, 2, \ldots, T\}$,
\begin{align} \label{eq:X-rec}
	X_{i, t} &= X_{i, t-1} + \one{A_t = i} - \one{C_t = i}.
\end{align}

The processing times of successive tasks at client~$i$ are assumed to be independent and exponentially distributed with rate $\mu_i > 0$, for each $i \in \{1, 2, \ldots, n\}$. In particular, for each $t \in \{1, 2, \ldots, T\}$, we have $\bP[C_t = i | X_{t-1}, A_t] \propto \mu_i$
for all $i \in \{1, 2, \ldots, n\}$ such that $X_{i, t-1} + \one{A_t = i} > 0$.
As we will show in later sections, this assumption enables the derivation of tractable performance bounds while still capturing queueing dynamics, phenomena that were not addressed in prior work before~\cite{leconte2024queueing}.

Critically, the gradient estimate $g_{C_t}(w_{I_t})$
returned at the end of an iteration~$t$
is based on the model parameters~$w_{I_t}$
known to the \gls{CS} at the beginning of round~$I_t = \sum_{s = 0}^t s \one{s + D_{A_s, s} = t}$
when the task was assigned to client~$C_t$,
where $D_{i, t}$ is called the \emph{relative delay}
and is defined as the number of rounds that get completed
during the sojourn of a task at a client:
\begin{align}
	\label{eq:Dit}
	D_{i, t} &= \one{A_t = i} R_{i, t},
	~\text{where} \\
	\label{eq:Rit}
	R_{i, t}
	&= \min\left\{
	r \in \bN:
	\sum_{s = t}^{t+r} \one{C_s = i}
	= X_{i, t-1} + \one{A_t = i}
	\right\}.
\end{align}

Consistent with the literature on decentralized learning, we make the following assumptions: a uniform lower bound $f^*$ on the objective function $f$; $L$-Lipschitz continuity of the gradients $\nabla f_i$ to ensure smoothness ($\|\nabla f_i(w) - \nabla f_i(\mu)\| \leq L \|w - \mu\|$); an upper bound $\sigma^2$ on the variance of the stochastic gradients ($\mathbb{E}[\|g_i(w) - \nabla f_i(w)\|^2] \leq \sigma^2$); and an upper bound $M^2$ on the gradient dissimilarity across clients ($\|\nabla f(w) - \nabla f_i(w)\|^2 \leq M^2$). These are detailed in Assumptions~\textbf{A1}--\textbf{A4} in Section~A of the supplementary material~\cite{alahyane2025optimizing}.

\subsection{Queueing dynamics} \label{sec:queueing}

The next result is a variant of \citep[Proposition~2]{leconte2024queueing}
and will be instrumental throughout the paper.
It relates the queueing dynamics to the routing and service-rate vectors $p = (p_1, p_2, \ldots, p_n)$ and $\mu = (\mu_1, \mu_2, \ldots, \mu_n)$.
Recall that the state space is $\cX_{n, m-1}$ (and not $\cX_{n, m}$)
because we consider model-parameter-update times.
Buzen's algorithm was introduced in \cite{buzen:1973}.

\begin{proposition} \label{prop:jackson} 
	In the framework of \Cref{sec:fl},
	the sequence $(X_t, t \in \bN)$
	defines an irreducible positive recurrent Markov chain
	with stationary distribution
	\begin{align} \label{eq:pi}
		\pi_{n, m-1}(x)
		&= \frac1{Z_{n, m-1}} \prod_{i = 1}^n \left( \frac{p_i}{\mu_i} \right)^{x_i},
		~ x \in \cX_{n, m-1},
	\end{align}
	where the normalizing constant~$Z_{n, m-1}$
	can be computed by applying Buzen's recursive algorithm:
	\begin{itemize}
		\item $Z_{\un, 0} = 1$
		for each $\un \in \{1, 2, \ldots, n\}$,
		\item $Z_{1, \um} = (\frac{p_1}{\mu_1})^\um$
		for each $\um \in \{0, 1, 2, \ldots, m\}$,
		\item $Z_{\un, \um} = Z_{\un-1, \um} + \frac{p_\un}{\mu_\un} Z_{\un, \um - 1}$
		for $\un \in \{2, 3, \ldots, n\}$
		and $\um \in \{1, 2, \ldots, m\}$.
	\end{itemize}
\end{proposition}
\begin{proof}
    See Section~\ref{app:jackson} in the supplementary material~\cite{alahyane2025optimizing}.
\end{proof}

In the rest of the paper,
we will assume that the system starts in steady state.
This assumption is reasonable
when the total number~$T$ of updates is sufficiently large,
as the distribution of~$X_t$
converges exponentially fast with~$t$
towards the stationary distribution
regardless of the initial distribution (see Theorem~13.4.14 in \cite{bremaud2013markov}).
Concretely, we will assume that
$X \triangleq X_0$ follows the stationary distribution~$\pi_{n, m-1}$,
and we will often drop the time index,
so that for instance $D_i$ will be a random variable
distributed like $D_{i, t}$ for any $t \in \bN_{> 0}$.

\section{OPTIMIZE MODEL UPDATES} \label{sec:optimize-updates}

Consistently with the literature on asynchronous \gls{FL}, in this section, we assume the \gls{CS} has a fixed budget~$T$ of model parameter updates, and we try to make the best of these updates. This is particularly relevant in contexts, such as cellular networks with costly data plans or satellite internet services, where data transmission costs are high.

\subsection{Bound, delay, and gradient descent} \label{sec:optimize-updates-delay}

\citep[Theorem~1]{leconte2024queueing} gave the following upper bound on the ergodic mean of the norm-square of the gradient of~$f$: there exists $\eta_{\text{max}} > 0$ dependent on~$p$ \footnote{The original proof of the bound~$G$ in~\cite{leconte2024queueing} incorrectly assumes independence between $D_{i,k}$ and $\nabla f(w_k)$ for all~$k \in \{1,\ldots,T\}$, which does not generally hold in asynchronous settings. We provide a correction based on a stochastic bound in Section~\ref{issue:eta} of the supplementary material~\cite{alahyane2025optimizing}, which yields a revised expression for $\eta_{\text{max}}$.} such that, for any $\eta \in (0, \eta_{\text{max}})$,
\begin{align}
	\nonumber
	&\frac{1}{T + 1}\sum_{t=0}^{T} \mathbb{E}[\|\nabla f(w_t)\|^2] \leq 8G, \text{ where} \\
	\label{eq:G}
	&\begin{aligned}
		G ={}
		\frac{A}{\eta(T + 1)}
		+ \frac{\eta L B}{n^2} \sum_{i=1}^{n} \frac{1}{p_i}
		+ \frac{\eta^2 L^2 B m}{n^2} \sum_{i=1}^{n} \frac{\bE[D_i]}{p_i^2}
		\text{,}
	\end{aligned}
\end{align} 
with $A=f(w_0) - f^*$ and $B = \sigma^2 + 2M^2$ (see Section~\ref{assumptions} of the supplementary material~\cite{alahyane2025optimizing} for details on these constants).
If needed, the dependency of~$G$ on the routing vector~$p$ will be made explicit by writing $G(p)$.
In the remainder, we use~$G$ as a proxy objective to minimize the ergodic norm-squared gradient of~$f$.


The first term in~$G$ follows a classical pattern, capturing the influence of initialization on convergence. The next two terms reflect the effects of stochastic gradient estimation (through~$\sigma$) and data heterogeneity across clients (through~$M$). The second term also captures the variance of gradient updates induced by the routing-dependent learning rate and is minimized by uniform routing. Anticipating the discussion in Section~\ref{sec:optimize-updates-discussion}, the third term quantifies gradient staleness via the relative delay~$\bE[D_i]$.
The last two terms in~$G$ depend heavily on the routing vector~$p$ and the service rate vector~$\mu$, both explicitly and implicitly through the mean relative delays~$\bE[D_i]$.

\textit{A priori}, expected relative delays are complex functions of the system dynamics, as a task's delay may depend on an unbounded number of future rounds.
\Cref{theo:little}, our first main contribution, bypasses this difficulty by expressing them as functions of the mean numbers of stationary tasks, which, in turn, allows us to derive closed-form expressions for the mean relative delays and their gradient.

\begin{theorem} \label{theo:little}
	In the framework of \Cref{sec:fl}, we have
	\begin{align}
		\label{eq:D}
		\bE[D_i]
		&= \bE[X_i],
		\quad i \in \{1, 2, \ldots, n\},
		\\
		\label{eq:gradD}
		\frac{\partial \bE[D_i]}{\partial p_j}
		&= \frac1{p_j} \cov[X_i, X_j],
		\quad i, j \in \{1, 2, \ldots, n\},
	\end{align}
	where for each $i, j \in \{1, 2, \ldots, n\}$,
	\begin{align}
		\label{eq:Xi}
		\bE[X_i]
		&= \sum_{k = 1}^{m-1}
		\left( \frac{p_i}{\mu_i} \right)^k
		\frac{Z_{n, m-1-k}}{Z_{n, m-1}}, \\
		\label{eq:XiXj}
		\bE[X_i X_j]
		&=
		\hspace{-.05cm}
		\sum_{\substack{k, \ell = 1 \\ k + \ell \le m - 1}}^{m-1}
		\hspace{-.05cm}
		\left( \frac{p_i}{\mu_i} \right)^{k}
		\left( \frac{p_j}{\mu_j} \right)^{\ell}
		\frac{Z_{n, m-1-k-\ell}}{Z_{n, m-1}},
	\end{align}
	and the constants $Z_{n, \um}$ for $\um \in \{0, 1, \ldots, m - 1\}$
	are computed as in \Cref{prop:jackson}.
\end{theorem}

\begin{proof}
	The proof of \Cref{theo:little} is in Section~\ref{app:little} of the supplementary material~\cite{alahyane2025optimizing}.
	Here we briefly give the intuition behind~\eqref{eq:D},
	which is analogous to Little's law.
	If each task at client~$i$ pays \$1
	each time a task gets completed at another client,
	there are two ways of collecting payments:
	either we receive an upfront payment of $\$ R_i$
	when a task is assigned to client~$i$ (yielding $\bE[D_i]$),
	or we earn $\$ X_i$
	each time a task gets completed at another client (yielding $\bE[X_i]$).
	Equations~\eqref{eq:gradD}--\eqref{eq:XiXj}
	follow by direct computations,
    after observing that the distribution~\eqref{eq:pi}
	is an exponential family.
\end{proof}
In the remainder, we will apply \eqref{eq:D}--\eqref{eq:gradD} to compute $\bE[D_i]$ and $\nabla_p \bE[D_i]$. However, these equations can also be used to estimate these quantities through Monte Carlo simulations.
Section~\ref{app:G} of the supplementary material~\cite{alahyane2025optimizing}
gives a gradient-descent algorithm to optimize the routing vector~$p$ in view of minimizing~$G$, which will be applied in \Cref{num:optimize-updates} to run extensive numerical results.

\subsection{Discussion} \label{sec:optimize-updates-discussion}

\Cref{theo:little} allows us to gain insight into
the impact on~$G$ of the system parameters.

\paragraph{How to minimize~$G$?}

One can verify that the second term in~$G$ is minimized by applying
the uniform routing~$p^{\text{uniform}}$,
given by $p^{\text{uniform}}_i = \frac1n$
for each $i \in \{1, 2, \ldots, n\}$.
\Cref{fig:G-third-term} shows that minimizing the third term in~$G$,
involving the mean relative delays,
is more challenging: even in a toy 2-client example,
the third term is non-monotonic
and is minimized by assigning almost all tasks to the slowest client.
\begin{figure}[htbp]
    \begin{center}
    \includegraphics[width=.47\textwidth]{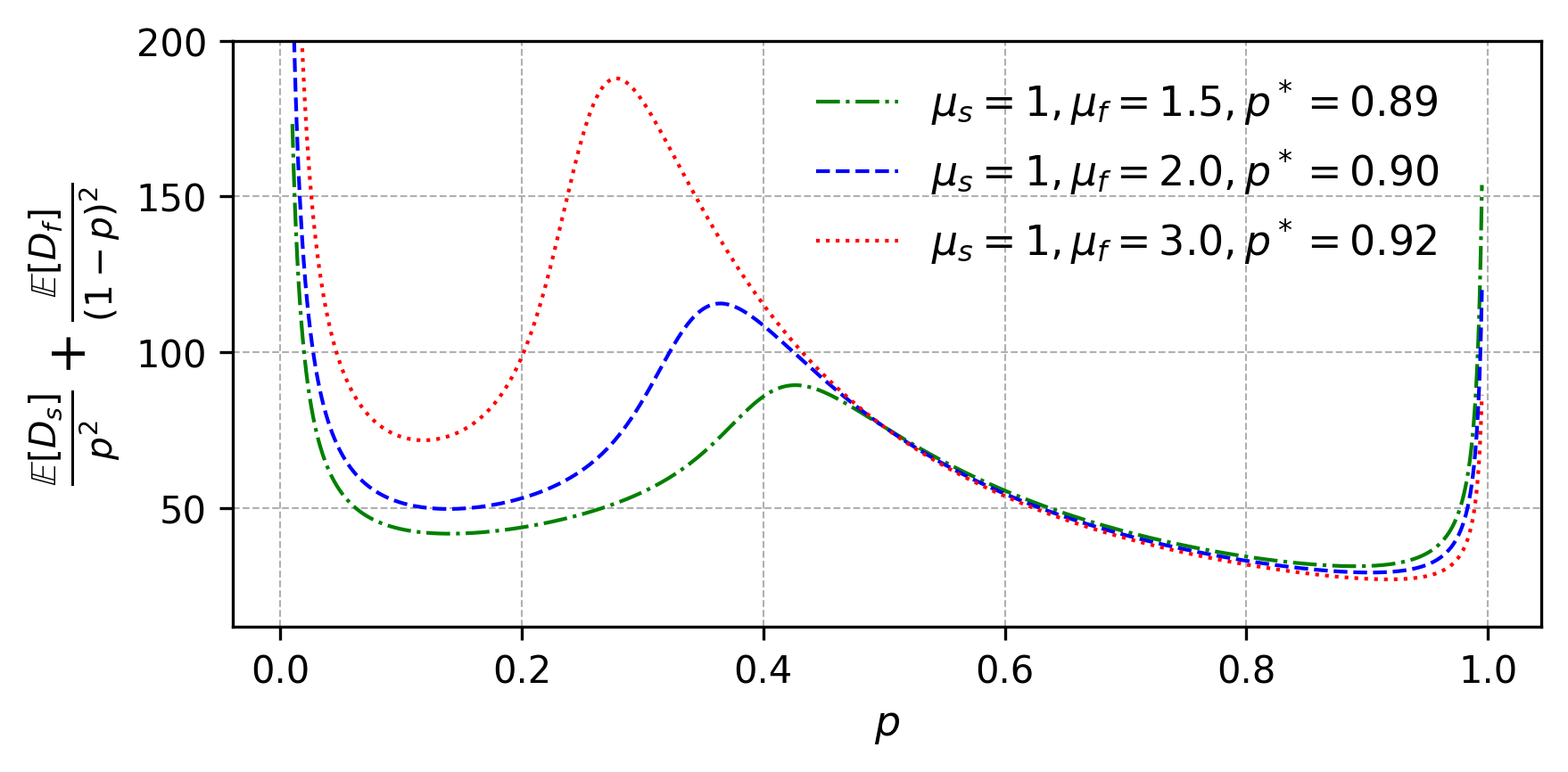}
    \caption{Third term of the bound~$G$ given in~\eqref{eq:G} vs.\ the routing probability to the slowest client, in a toy example with $n = 2$ clients and $m = 20$ tasks, for various speed vectors~$\mu=\left(\mu_s, \mu_f \right)$.}
    \label{fig:G-third-term}
    \end{center}
\end{figure}

To gain more insight into this third term,
observe that Equation~\eqref{eq:D} from \Cref{theo:little}
yields the following simple relation between the mean relative delays:
\begin{align} \label{eq:tot}
	\sum_{i = 1}^n \bE[D_i]
	= \sum_{i = 1}^n \bE[X_i]
	= m - 1.
\end{align}
This relation has several consequences, in particular:
(i) the sum of the mean relative delays depends only on the numbers~$n$ and~$m$ of clients and tasks, while~$p$ and~$\mu$ impact only how the relative delay is distributed across clients;
(ii) decreasing the relative delay at a client necessarily comes at the cost of an increased relative delay at another client;
and (iii) since the routing probabilities $p_i$ and relative delays $\bE[D_i]$ both have constant sums over the clients~$i$, minimizing the third term in~$G$ requires finding a vector~$p$ so that a client~$i$ with a high relative delay $\bE[D_i]$ also has a relatively large routing probability~$p_i$.

\paragraph{Dependency on the number~$m$ of tasks}

Another consequence of \Cref{theo:little}
is that the bound~$G$
is an increasing function of the number~$m$ of tasks
(by combining~\eqref{eq:D}
with the observation that $\bE[X_i]$ is a non-decreasing function of~$m$ \citep[Lemma~2]{suri1985monotonicity}).
In particular,
keeping all other system parameters fixed,
the performance is optimized when only~$m = 1$ task circulates in the network!
In this case, \eqref{eq:tot} implies
that the third term in~$G$ is equal to zero.
This is intuitive because, with a single task circulating in the network, the staleness issue is trivially eliminated, and the system works like a synchronous \gls{FL} system in which the \gls{CS} samples a single client at each round. This observation motivates the alternative metric we introduce in \Cref{sec:optimize-time}.

\paragraph{Simple routing strategies}
Our result allows us to simplify the bound~$G$ for two noteworthy routing vectors, that serve as baselines in our experiments.
First, under uniform routing $p^{\text{uniform}}$,
\eqref{eq:tot} yields
\begin{align*}
G(p^{\text{uniform}})
&= \frac{A}{\eta (T + 1)}
+ \eta L B
+ \eta^2 L^2 B m (m-1).
\end{align*}
Note that \texttt{Generalized AsyncSGD} with uniform routing reduces to the standard \texttt{AsyncSGD} algorithm.
\begin{figure*}[ht]
\begin{center}
\centerline{\includegraphics[width=\textwidth]{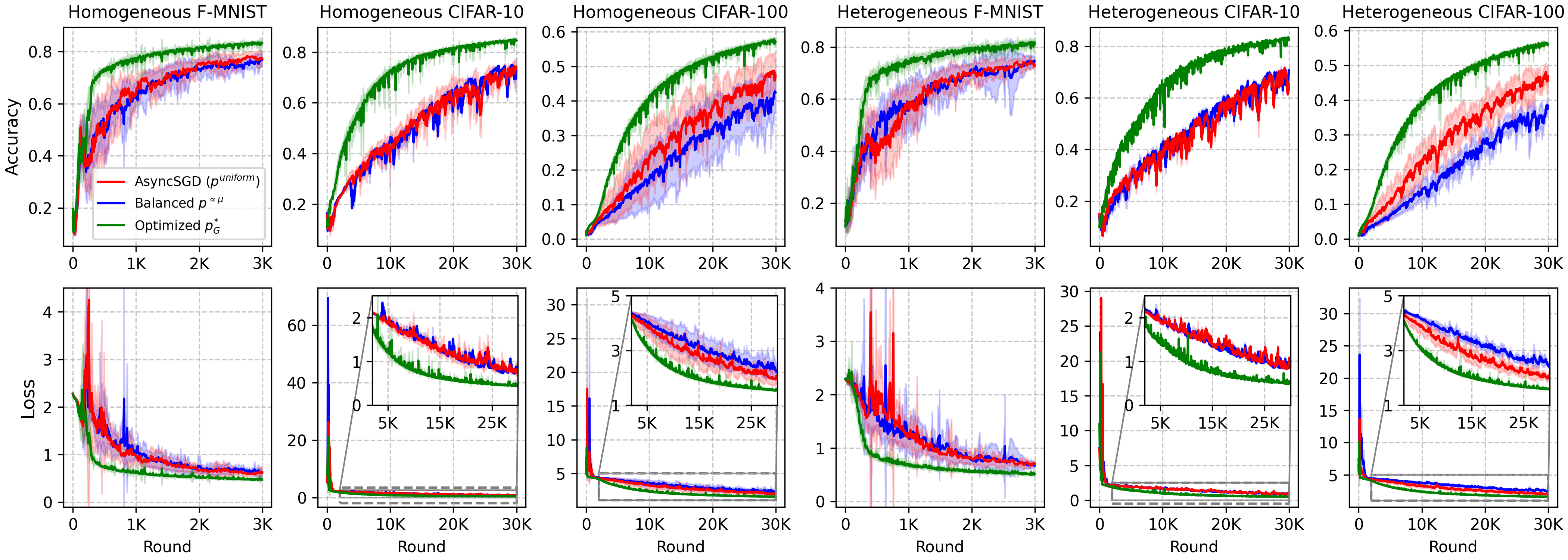}}
\caption{Performance on the test set at the \gls{CS} in the scenario of \Cref{num:optimize-updates}, with $n = 20$ clients and $m = 100$ tasks, under homogeneous and heterogeneous data distributions. Solid lines show averages over independent runs; shaded areas denote standard deviations. For Fashion-MNIST, we ran 10 simulations of 3,000 rounds, recording accuracy and loss every 5 rounds. For CIFAR-10 and CIFAR-100, we applied standard normalization and data augmentation, and ran 3 simulations of 30,000 rounds, logging performance every 50 rounds. All routing strategies use the same model initialization.}
\label{fig:sim_results}
\end{center}
\end{figure*}

Another routing strategy that admits explicit analysis is the \emph{balanced} routing vector $p^{\propto \mu}$, defined by
$p^{\propto \mu}_i = \frac{\mu_i}{\sum_j \mu_j}$,
i.e., each client is selected with probability proportional to its service speed. This strategy is well-known in queueing theory for ``balancing the load'' across clients, in the sense that $\mathbb{E}[D_i] = \mathbb{E}[X_i] = (m - 1)/n$ for all $i \in \{1, 2, \ldots, n\}$ (see \eqref{eq:pi}, \eqref{eq:D}, and \eqref{eq:tot}). Under this policy, all clients experience the same average relative delay, regardless of 
individual speeds. Moreover, this strategy is commonly used as a heuristic to maximize the throughput of closed Jackson networks.
It follows that
\begin{align*}
G(p^{\propto \mu})=
\frac{A}{\eta (T + 1)}
+ \frac{\eta L B |\mu|}{n^2} \sum_{i = 1}^n \frac1{\mu_i}\left(1
+ \frac{\eta L m (m-1) |\mu|}{n \mu_i}\right).
\end{align*}

\subsection{Numerical results} \label{num:optimize-updates}

In this section, we numerically evaluate the impact of the optimizations proposed in \Cref{sec:optimize-updates-delay} on the accuracy and loss performance of \texttt{Generalized AsyncSGD} across several image classification tasks. The objective function $G$ is used as a proxy to minimize the average squared norm of the gradient of~$f$ per update round.
We consider mainly two data-distribution scenarios:

\textbf{Homogeneous}:
Data is distributed independently and identically across clients,
and each client receives an equal number of data points from each class.

\textbf{Heterogeneous}:
Datasets are heterogeneous across clients, both in terms of distribution and volume. For each class~$k$, we sample a vector $q_k \sim \text{Dir}_n(0.5)$, where $q_{k,j}$ is the proportion of class-$k$ instances allocated to client~$j$ and $\text{Dir}_n(\beta)$ the Dirichlet distribution with dimension $n$ and concentration parameter $\beta > 0$ \cite{li2021federatedlearningnoniiddata}.

Given a system as described in \Cref{sec:fl}, the $G$-optimized routing vector $p^*_G$ is computed by minimizing $G$ using the Adam gradient-descent algorithm with standard hyperparameters 
from the literature \cite{kingma2014adam}, initialized with the uniform routing vector~$p^{\text{uniform}}$. Since~$G$ is non-convex, the resulting routing vector may be a local minimum. The gradient of~$G$ is computed in closed form using the results of \Cref{theo:little}, as detailed in Section~\ref{app:G} of the supplementary material~\cite{alahyane2025optimizing}.


Given a routing vector~$p$, we simulate the system of \Cref{sec:fl} and evaluate \texttt{Generalized AsyncSGD}
on image classification tasks using the
Fashion-MNIST \cite{deng2012mnist}, CIFAR-10, and CIFAR-100 \cite{Krizhevsky2009LearningML} datasets. We employ the standard multi-class cross-entropy loss and evaluate performance on an unseen, label-balanced test dataset. Additional experimental details are provided in Section~\ref{exp:dl} of the supplementary material~\cite{alahyane2025optimizing}.

We consider a network of $n = 20$ clients managing $m = 100$ tasks. For each $i \in \{1, \ldots, 20\}$, the service speed of client~$i$ is set to $\mu_i = \exp(i/100)$, such that the fastest client is approximately 20\% faster than the slowest. The learning rate is $\eta=0.01$, and $L=1$. As discussed in \Cref{sec:optimize-updates-discussion}, this high-concurrency setting poses a significant challenge for minimizing gradient staleness.
To assess the robustness of $p^*_G$ to the stationarity assumption, we initialized the system out of stationarity by assigning $\frac{m}{n}$ tasks to each client instead of sampling the initial state from the stationary distribution.
We compare the performance of \texttt{Generalized AsyncSGD} under the optimized routing vector $p^*_G$ with two baselines:
(i) \texttt{AsyncSGD}, corresponding to \texttt{Generalized AsyncSGD} with the uniform routing vector $p^{\text{uniform}}$;
(ii) \texttt{Generalized AsyncSGD} with the balanced routing vector~$p^{\propto \mu}$ described in the previous section.

\paragraph{Optimizing~$G$.}

The optimized vector~$p^*_G$ obtained by minimizing~$G$ selects the slowest client over 40\% of the time, while the routing probabilities to faster clients are all of the same magnitude but decrease as the client speed increases. This result is consistent with \citep[Section~5]{leconte2024queueing} but counterintuitive \textit{a priori}. Such a skewed routing vector is obtained because it significantly reduces the third term of~$G$ (see \Cref{sec:optimize-updates-discussion}). Intuitively, $p^*_G$ synchronizes the system to the pace of the slowest client, while routing tasks to faster clients in inverse proportion to their speeds, to reduce errors caused by stale gradients.

Additional numerical results (not shown here) reveal that the skewness of the optimized routing vector~$p^*_G$ towards the slowest client is accentuated when the concurrency~$m$ is large, as in the scenario we consider here. More specifically, the vector~$p^*_G$ is closer to the uniform vector~$p^{\text{uniform}}$ when
$m$
is small relative to the number~$n$ of clients, while the routing probability to the slowest client rises significantly as $m$ increases.
This is in line with
\Cref{sec:optimize-updates-discussion},
where we observed that the relative weight of third term of~$G$ tends to increase with~$m$.

\paragraph{Performance on datasets.}
\Cref{fig:sim_results} shows that \texttt{Generalized AsyncSGD} with the optimized routing strategy \( p^*_G \) consistently outperforms the baseline methods, namely, \texttt{AsyncSGD} and the balanced routing strategy \( p^{\propto \mu} \), across all experiments and throughout the learning process. Notably, although assigning a higher routing probability to slower clients might seem to skew the model toward their local data distributions, our simulations show that \( p^*_G \) achieves both superior and more stable performance. This robustness is partly due to the use of an adaptive learning rate, which ensures that gradient updates remain unbiased despite routing asymmetries. In contrast, the standard deviation of the loss is significantly higher under uniform and balanced strategies, which we attribute to increased gradient staleness.


Additional experiments in Section~\ref{app:num-dist-heterogeneity} of the supplementary material~\cite{alahyane2025optimizing} confirm that the results hold under more heterogeneous data distributions, where clients have highly imbalanced and disjoint label sets, i.e., each client’s dataset contains only a subset of image labels, possibly disjoint from those of other clients. Further experiments with deterministic and lognormal (heavy-tailed) computation times, reported in Section~\ref{app:num-compu-time-heterogeneity} of the supplementary material~\cite{alahyane2025optimizing}, show similar performance trends, indicating robustness to the assumption of exponentially distributed computation times.


Looking ahead to \Cref{sec:optimize-time}, note that the performance advantage of \( p^*_G \) is specific to the round-based metric used in \Cref{fig:sim_results}. If the x-axis represented wall-clock time instead of update rounds, the ranking would be reversed. As discussed in \Cref{sec:optimize-updates-discussion}, minimizing \( G \) reduces staleness by favoring slower clients, but comes at the cost of lower throughput.
Concretely, in \Cref{fig:sim_results}, the average wall-clock time to complete 3,000 update rounds was 7 to 8 times larger under $p^*_G$ compared to $p^{\text{uniform}}$ and $p^{\propto \mu}$. We believe this trade-off may be prohibitive in practice, even when wall-clock time is not the primary concern. This motivates the development of alternative metrics that better balance staleness reduction and update frequency to optimize convergence in terms of wall-clock time rather than round count.

\section{OPTIMIZE  WALL-CLOCK TIME} \label{sec:optimize-time}

In this section, we aim to optimize performance with respect to wall-clock time, explicitly accounting for the duration of model-parameter update rounds. In \Cref{sec:optimize-time-delay}, we introduce a new performance metric that incorporates these durations and derive an upper bound~$H$, which serves as the wall-clock-time counterpart to~$G$. Additionally, we provide a proxy for the expected wall-clock time required to reach $\epsilon$-accuracy.
\Cref{sec:optimize-time-discussion} highlights the key differences between this analysis and that of \Cref{sec:optimize-updates}. Finally, in \Cref{num:optimize-time}, we present numerical results demonstrating improved model performance in wall-clock time when optimizing $H$ rather than $G$.

\subsection{Bound, delay, and gradient descent} \label{sec:optimize-time-delay}

\Cref{theo:bound-time} below provides an upper bound on the ergodic mean of the squared norm of the gradient of $f$, weighted by the expected duration of each corresponding round. This quantity can be interpreted as a per-round cost, where the cost is defined as the product of the squared gradient norm and the average duration of that round.
For each $t \in \{0, 1, \dots, T\}$, let $\tau_t$ denote the duration of round $t$ in the model described in \Cref{sec:fl}, and define $\bar{\tau}_t = \mathbb{E}[\tau_t]$.

\begin{proposition} \label{theo:bound-time}
In the framework of \Cref{sec:fl}, there exists $\eta_{\text{max}}>~0$ such that for all $\eta \in (0, \eta_{\text{max}})$, the following bound holds:
\begin{align} \label{boundH}
     \frac{1}{T+1}\sum_{t=0}^{T} \bar{\tau}_t \bE \left[ \|\nabla f(w_t)\|^2 \right] \leq 8H = 8 \frac{G}{\lambda}.
\end{align}
Here, $\lambda$ is the throughput of the Jackson network, that is, the number of rounds per unit of  wall-clock time, given as follows, with $Z_{n, m}$ and $Z_{n, m-1}$ as defined in \Cref{prop:jackson} and $\xi \sim \pi_{n, m}$:
	\begin{align}
        \label{eq:throughput}
        \lambda = \sum_{i=1}^n \mu_i \mathbb{P}(\xi_i > 0) = \frac{Z_{n,m-1}}{Z_{n,m}}.
    \end{align}
\end{proposition}
\begin{proof}
	See Section~\ref{app:time} of the supplementary material~\cite{alahyane2025optimizing}.
\end{proof}

The bound $H$ in \Cref{theo:bound-time} can be viewed as the throughput-aware counterpart to the bound $G$. As discussed in the previous section, minimizing $G$ alone (which reflects the average model error) tends to route most tasks to slower clients, significantly reducing system throughput. This occurs because prioritizing slow clients helps reduce staleness, thereby improving model accuracy per iteration.

In contrast, the quantity $H = G / \lambda$ captures the trade-off between two competing objectives: minimizing $G$ to improve per-iteration model quality, and maximizing $\lambda$ to increase the frequency of updates in wall-clock time. These two goals are often in conflict, as reducing one typically increases the other. Thus, $H$ provides a principled objective that balances staleness reduction with practical training speed, making it more realistic in asynchronous FL settings.


Another useful metric for evaluating performance over wall-clock time is the expected time to reach $\epsilon$-accuracy. \Cref{prop:time-eps} gives its expression within the framework of \Cref{sec:fl}. Let $\tilde{\tau}_T = \sum_{t = 0}^{T-1} \tau_t$ denote the time required to complete $T$ rounds.

\begin{proposition}[Time to achieve an $\epsilon$-accuracy]
\label{prop:time-eps}
Assume that the learning rate $\eta = \frac{C}{T^\alpha}$ is a function of the number of rounds $T$, where $\alpha, C \in \bR_{> 0}$, such that $\alpha < 1$ and $\eta < \eta_{\text{max}}$. Then it holds that $\frac{1}{T+1} \sum_{t=0}^{T} \mathbb{E}[\|\nabla f(w_t)\|^2] \leq \epsilon$ whenever $T \ge T_\epsilon$, where
\begin{align*}
    T_{\epsilon} = \mathcal{O}\Bigg( 
        \left( \frac{A}{C\epsilon} \right)^{\frac{1}{1-\alpha}}
    &+ \left( \frac{C L B}{\epsilon n^2} \sum_{i=1}^{n} \frac{1}{p_i} \right)^{\frac{1}{\alpha}} \\
        &+ \left( \frac{C^2 L^2 B m}{\epsilon n^2} \sum_{i=1}^{n} \frac{\mathbb{E}[D_i]}{p_i^2} \right)^{\frac{1}{2\alpha}} 
    \Bigg)
\end{align*}
rounds with expected wall-clock time
$\mathbb{E}[\tilde{\tau}_{T_\epsilon}] = T_{\epsilon}/\lambda$.
\end{proposition}
\begin{proof}
	See Section~\ref{proof:time_eps} of the supplementary material~\cite{alahyane2025optimizing}.
\end{proof}
Each term in $T_\epsilon$ coincides with a term in~$G$. The metric $\mathbb{E}[\tilde{\tau}_\epsilon] = T_{\epsilon}/\lambda$
further highlights the trade-off between minimizing per-round error and maximizing update frequency. Minimizing the number of rounds~$T_\epsilon$ to reach $\epsilon$-accuracy may slow the system by favoring slower clients to reduce staleness, while maximizing throughput~$\lambda$ increases update frequency but can worsen staleness. Like the proxy~$H$, the expected time to reach $\epsilon$-accuracy captures this fundamental trade-off.

\subsection{Discussion} \label{sec:optimize-time-discussion}
Part of the discussion in~\Cref{sec:optimize-updates-discussion}
can be adapted to~$H$ with minor modification.
One fundamental difference between~$G$ and~$H$ is that $H$ is generally \emph{not} a non-decreasing function of the number~$m$ of tasks. Figure~\ref{fig:optimal_m} shows the bound $H(p^{\text{uniform}})$ as a function of~$m$ under different step sizes~$\eta$, revealing the existence of an optimal number~$m^* > 1$ of tasks that minimizes~$H$. Our intuition is that when $m < m^*$, throughput is insufficient and clients are underutilized; conversely, when $m > m^*$, throughput increases but so does staleness, which ultimately hinders convergence, as already observed with~$G$.
As a side remark, observe that the optimal number~$m^*$ of tasks decreases (albeit slowly) with~$\eta$. Our intuition is that, with a higher~$\eta$, the impact of stale gradients becomes more significant, as each gradient carries more weight in the model-parameter updates.
\begin{figure}[htbp]
	\begin{center}
	\includegraphics[width=0.43\textwidth]{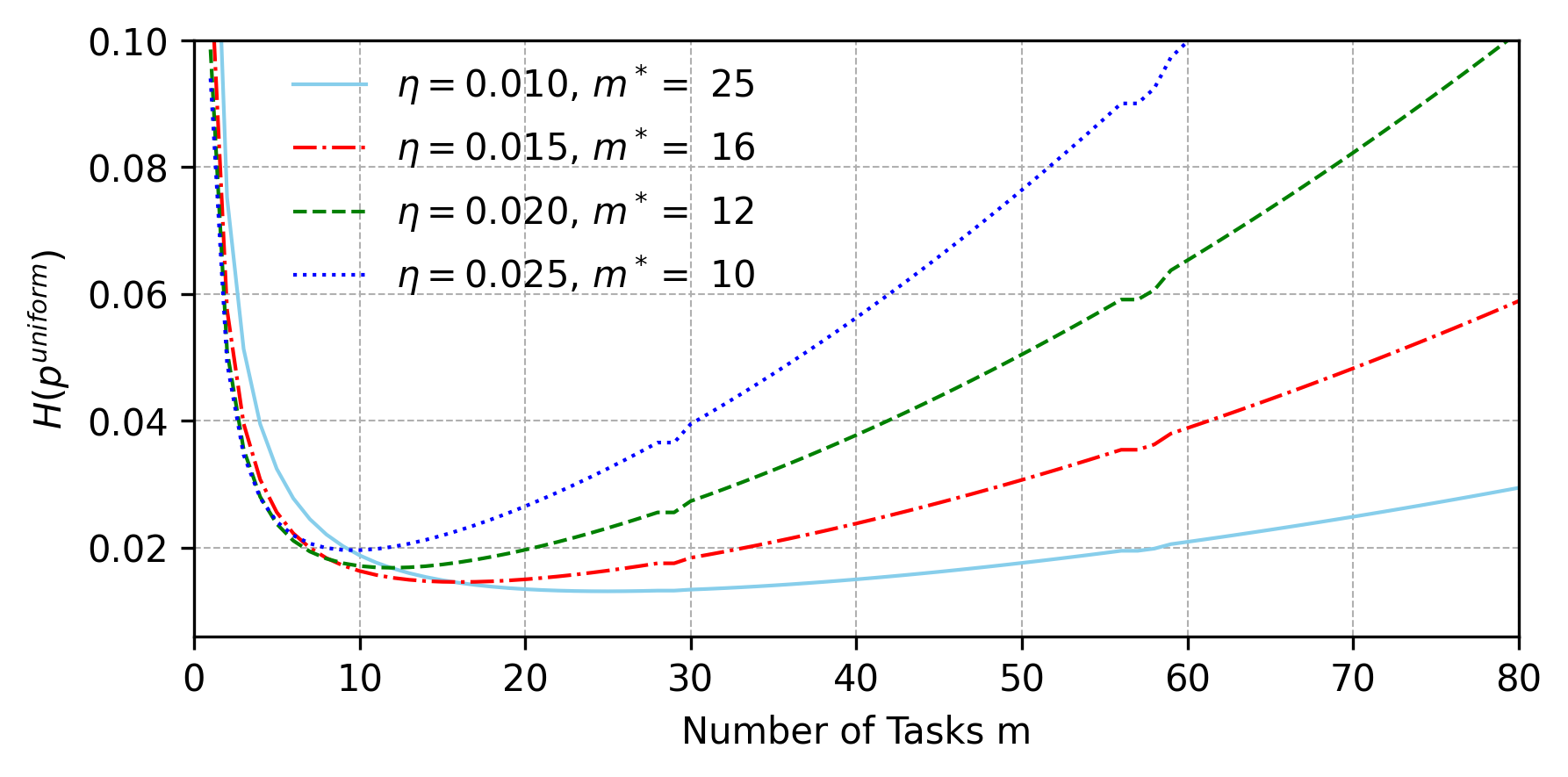} 
	\caption{Bound~$H(p^{\text{uniform}})$ as a function of the number~$m$ of tasks for different values of the step size~$\eta$. The system consists of 50 clients, with speeds given by $\mu_i=\exp(i/100)$ for each $i \in \{1,\ldots,n\}$.}
	\label{fig:optimal_m}
    \end{center}
\end{figure}
\subsection{Numerical results} \label{num:optimize-time}

In this section, we follow the same experimental procedure as in \Cref{num:optimize-updates}, but now we optimize the performance metric~$H$ introduced in \Cref{sec:optimize-time-delay}, which accounts for wall-clock time. The routing vector~$p^*_H$ is obtained by minimizing~$H$ using the same Adam-based optimization setup described previously. We evaluate the performance of \texttt{Generalized AsyncSGD} under this routing strategy on the KMNIST dataset~\cite{clanuwat2018deep}, using the same simulation and evaluation framework as in the first set of experiments. 

We analyze a network of $n = 30$ clients with concurrency level $m = 30$. Clients are organized into three clusters of 10. The slowest cluster has an average service time of 100 time units, the medium cluster 10 time units, and the fastest 1 time unit. This low-concurrency, high-speed-heterogeneity scenario is particularly challenging for optimizing throughput. We set the parameters as follows: $L = 1$, $\sigma = 3$, $M = 10$, $A/T = 15$, and $\eta = 0.01$. More details appear in Sections~\ref{app:H} and~\ref{exp:dl} of the supplementary material~\cite{alahyane2025optimizing}.

\paragraph{Optimizing~$H$}

The $H$-optimized routing probabilities are (per client) $p^*_\text{H, \text{slow}}=0.0068$, $p^*_\text{H, \text{medium}}=0.0449$, and $p^*_\text{H, \text{fast}}=0.0487$. In contrast to \Cref{num:optimize-updates}, faster clients receive a larger fraction of tasks, but $p^*_{H, \text{medium}}$ and $p^*_{H, \text{fast}}$ remain of the same order, so that $p^*_H$ seems to achieve a trade-off between minimizing staleness (like $p^*_G$) and maximizing throughput (like $p^{\propto \mu}$, a common heuristic to maximize throughput). Concretely, within the  wall-clock time frame of 3{,}000 units plotted in \Cref{fig:sim_time_results}, $p^{\propto \mu}$ completes 17{,}000
rounds, which sets it apart from $p^*_H$ (3{,}200 rounds), $p^{\text{uniform}}$ (690), and $p^*_G$ (145).

\begin{figure}[tbp]
	\begin{center}
	\includegraphics[width=0.48\textwidth]{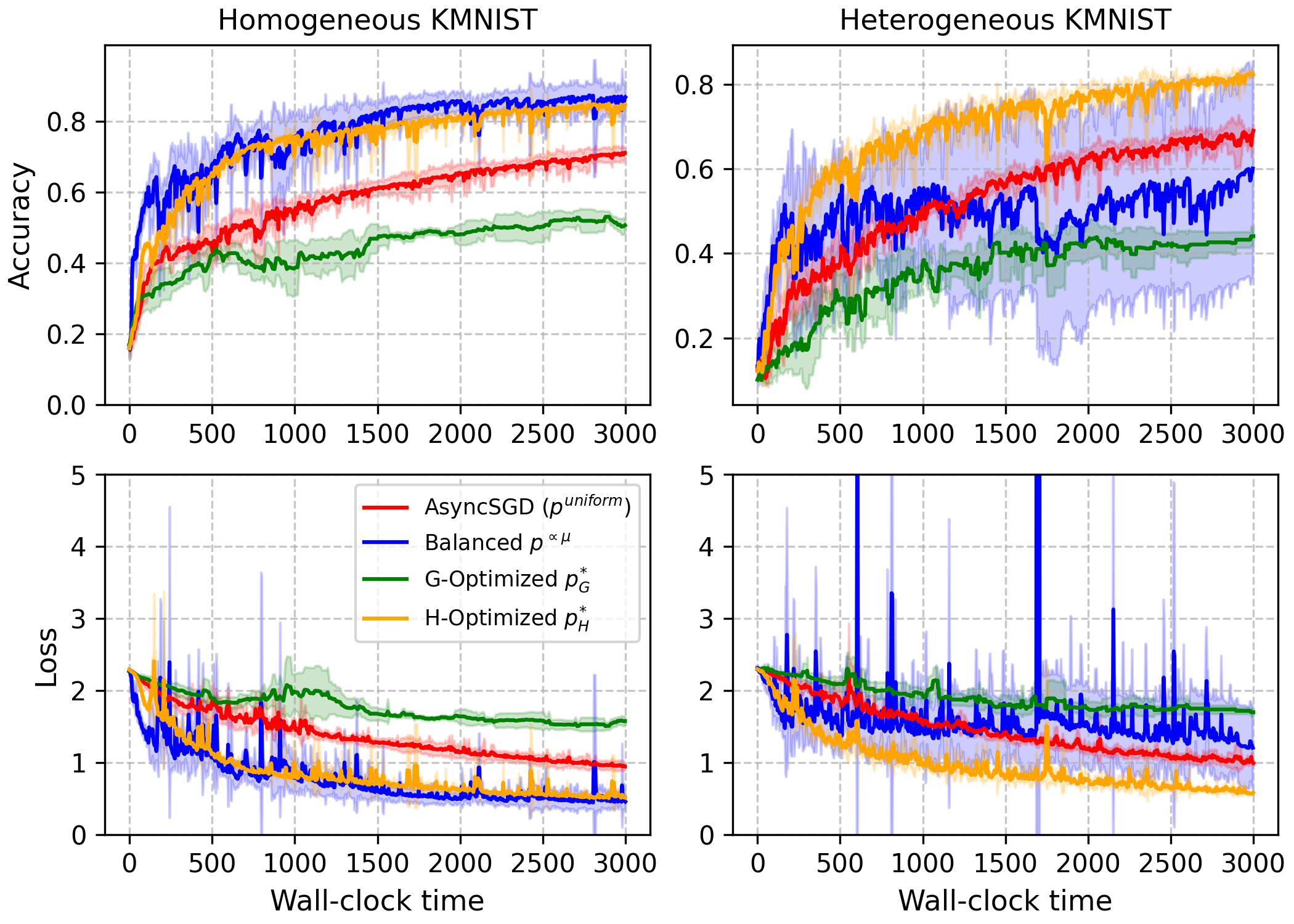} 
    \caption{Performance on the test set with respect to wall-clock time at the \gls{CS} in the scenario of \Cref{num:optimize-time}, with $n = 30$ clients and $m = 30$ tasks on homogeneous and heterogeneous KMNIST datasets. Simulations ran for 3{,}000 wall-clock time units and were repeated 10 times. Solid lines show means; shaded areas indicate standard deviations. \vspace{.3cm}}
	\label{fig:sim_time_results}
    \end{center}
\end{figure}

\paragraph{Performance on datasets}

Figure~\ref{fig:sim_time_results} compares the accuracy and loss of \texttt{Generalized AsyncSGD} under these four routing strategies.
Focusing first on the average performance,
we observe that in the homogeneous scenario,
$p^{\propto \mu}$ performs slightly better than $p^*_H$,
which in turn outperforms $p^*_G$ and $p^{\text{uniform}}$,
while in the heterogeneous scenario
$p^*_H$ outperforms all strategies.
Furthermore, the balanced strategy $p^{\propto \mu}$
exhibits sharp spikes in the loss trajectory in all scenarios,
and a particularly large standard deviation in the heterogeneous scenario.
This contrasts with the low and stable standard deviation
exhibited by both $p^*_G$ and $p^*_H$.
All in all, these numerical results confirm
that when wall-clock time is a performance criterion,
the $H$-optimized strategy provides a suitable trade-off between
minimizing staleness (overemphasized by $p^*_G$)
and maximizing throughput (overemphasized by $p^{\propto \mu}$), even though the bound $H$ serves only as a proxy for actual performance.
Section~\ref{app:num-additional-H} of the supplementary material~\cite{alahyane2025optimizing} includes additional plots for alternative image classification tasks, heterogeneous data distributions, and computation time distributions.

\section{CONCLUSION}


We provide novel insights into the impact of queueing dynamics on asynchronous \gls{FL}, enabling both performance optimization and the identification of fundamental limitations in existing performance objectives. Our experiments on real datasets show that queueing effects can significantly affect performance, and that our proposed optimizations can improve accuracy by 10\% to 30\%. These insights are relevant not only to \texttt{Generalized AsyncSGD}, our focus here, but also to algorithms such as \texttt{FedBuff}, which introduce additional control dimensions and could similarly benefit from our stochastic modeling framework. For example, we believe that adding a buffering mechanism at the \gls{CS} to aggregate gradients every $K$ service completion could reduce the average relative delay by a factor of $\frac{1}{K}$ under similar network dynamics. Modeling client unavailability is also a promising direction for future work.




\newpage

\begin{ack}
This research was supported in part by ANR EPLER, Projet E2CC, and the ``Data Science \& Processus Industriels'' chair funded by École Polytechnique, the Mohammed VI Polytechnic University, and Fondation de l’X. This research was also facilitated by the support of LabEx CIMI via the SOLACE project-team and the “Stochastic control and learning for complex networks” thematic semester.
\end{ack}

\bibliography{paper}

\onecolumn
\newpage
\twocolumn


\newpage
\appendix
\onecolumn

\begin{center}
\textbf{\\[-0.4cm] {\Large Optimizing Asynchronous Federated Learning: A Delicate Trade-Off Between} \\[0.2cm] {\Large Model-Parameter Staleness and Update Frequency} \\[0.3cm] {\large Supplementary Materials}}
\par\noindent\rule{\textwidth}{0.4pt}
\end{center}


\section{Assumptions} \label{assumptions}
Our analysis is grounded in the assumptions presented in \Cref{sec:fl}, which align with those established in \cite{leconte2024queueing}. These assumptions are detailed as follows:
\begin{enumerate}[label={\bfseries A\arabic*}]
    \item \label{A1} \textbf{Lower Boundedness:} The objective function $f$ is bounded from below by some real number $f^*$, meaning $f(w) \geq f^*$ for all $w \in \mathbb{R}^d$.
    
    \item \label{A2} \textbf{Gradient Smoothness:} Each client's local function $f_i$ has an $L$-Lipschitz continuous gradient, where $L > 0$. Mathematically, for any vectors $w, \mu \in \mathbb{R}^d$:
    \[
    \|\nabla f_i(w) - \nabla f_i(\mu)\| \leq L \|w - \mu\|.
    \]
    
    \item \label{A3} \textbf{Stochastic Gradient Properties:} For each client $i$, the stochastic gradient $g_i(w)$ is an unbiased estimator of the gradient $\nabla f_i(w)$ with bounded variance $\sigma^2 > 0$. That is, for all $w \in \mathbb{R}^d$:
    \[
    \mathbb{E}[g_i(w)-\nabla f_i(w)] = 0 \text{ and }
    \mathbb{E}[\|g_i(w) - \nabla f_i(w)\|^2] \leq \sigma^2.
    \]
    
    \item \label{A4} \textbf{Bounded Client Heterogeneity:} There exist constant $M > 0$, such that for all $w \in \mathbb{R}^d$:
    \[
    \|\nabla f(w) - \nabla f_i(w)\|^2 \leq M^2.
    \]

    \item \label{A5} \textbf{Stationary Regime:} The dynamics of the closed Jackson network are assumed to be in their stationary regime. Specifically, the random $ n $-dimensional vector $ X_t $, where $t \in \bN$, follows its stationary distribution $ \pi_{n,m-1}$.

\end{enumerate}

\section[Proof of Proposition~\ref{prop:jackson}]{Proof of \Cref{prop:jackson}} \label{app:jackson}

The result follows by observing that
$(X_t, t \in \bN)$ tracks the state at departure times
of a Jackson network~\cite{jackson1957networks,serfozo:1999}
with $n$ clients (usually called \textit{servers})
and $m$ tasks (often called \textit{customers} or \textit{jobs}),
with service rate vector~$\mu$
and relative arrival rate vector~$p$.
In particular, \eqref{eq:pi} follows from
Chapter~1 (Definition~1.8 and Theorem~1.12)
and Section~4.8 (Definition~4.34 and Example~4.38)
in \cite{serfozo:1999}.

\section[Revised Expression for ηₘₐₓ]{Revised Expression for $\eta_{\text{max}}$} \label{issue:eta}
In \cite{leconte2024queueing}, there is a technical issue in the proof establishing the upper bound~$G$. The authors assume that for all $k \in \{0,\ldots,T\}$, the relative delay $D_{i,k}$ is independent of $\nabla f(w_k)$. However, this assumption does not hold in the case of closed Jackson network dynamics and leads to an incorrect expression for $\eta_\text{max}$. 

To address this issue, we propose a stochastic upper bound on the relative delay $D_{i,k}$ using another random variable that is independent of $\nabla f(w_k)$, and we derive a corrected expression for $\eta_\text{max}$. Specifically, we consider the worst-case sojourn time scenario for the task sent to client~$i$ at round~$k$, which occurs when this task finds $m-1$ tasks already queued at client~$i$. 

Due to the memoryless property of exponential service times, the sojourn time of this task at client~$i$ is then distributed as an Erlang random variable $Er$ with parameters~$m$ and~$\mu_i$. The quantity $D_{i,k}$, representing the number of service completions during the sojourn of this task, can be stochastically bounded by the number of events generated by an independent Poisson process~$N$ with intensity~$\sum_{j=1}^n \mu_j$ over a time interval distributed as an Erlang with parameters~$m$ and~$\mu_i$. 

Therefore, for all $k \in \{0,\ldots,T\}$:
\begin{align*}
    \mathbb{E}[D_{i,k} \|\nabla f(w_k)\|^2]
    &\leq \mathbb{E}[N(Er) \|\nabla f(w_k)\|^2] \\
    &= \mathbb{E}[N(Er)] \cdot \mathbb{E}[\|\nabla f(w_k)\|^2] \\
    &= \mathbb{E}[\|\nabla f(w_k)\|^2] \cdot \frac{m}{\mu_i} \sum_{j=1}^n \mu_j,
\end{align*}
where $N(Er)$ denotes the number of events in a Poisson process with rate $\sum_{j=1}^n \mu_j$ over a time period following an independent Erlang distribution with parameters $(m, \mu_i)$.

Using this upper bound and following the structure of the proof in~\cite{leconte2024queueing}, we derive the following corrected expression for $\eta_{\text{max}}$:
\begin{align*}
    \eta_{\text{max}} = \frac{1}{4L} \min \left\{ 
    \left( \frac{m^2}{n^2} \sum_{j=1}^n \mu_j \sum_{i=1}^n \frac{1}{\mu_i p_i^2} \right)^{-1/2},
    \frac{2}{\sum_{i=1}^n \frac{1}{n^2 p_i}} 
    \right\}.
\end{align*}

\section[Proof of Theorem~\ref{theo:little}]{Proof of \Cref{theo:little}} \label{app:little}

Consider the asynchronous federated learning framework of \Cref{sec:fl}.
After proving preliminary results in
\Cref{lem:Rit} and \Cref{coro:Dit-bijection},
we prove \eqref{eq:D} in \Cref{app:D}
and \eqref{eq:gradD} in \Cref{app:gradD}.

\begin{lemma} \label{lem:Rit}
	For each $i \in \{1, 2, \ldots, n\}$, for each $t \in \bN_{> 0}$ we have
	\begin{align*}
		R_{i, t+1}
		&= \begin{cases}
			\max(0, R_{i, t} - 1)
			&\text{if $A_{t+1} \neq i$,} \\
			\min\{r \ge R_{i, t}: C_{t + 1 + r} = i\}
			&\text{if $A_{t+1} = i$.}
		\end{cases}
	\end{align*}
\end{lemma}

\begin{proof}
	Let $i \in \{1, 2, \ldots, n\}$ and $t \in \bN_{> 0}$.
	
	The following preliminary results stem
	from the definition~\eqref{eq:Rit} of $R_{i, t}$
	and the observation that
	the sequence $r \in \bN \mapsto \sum_{s = t}^{t+r} \one{C_s = i}$
	is nondecreasing with increments 0 or 1
	and takes value $\one{C_t = i} \in \{0, 1\}$ at $r = 0$
	(equal to $0$ if $X_{i, t-1} + \one{A_t = i} = 0$ by definition of $C_t$):
	\begin{itemize}
		\item We can replace the equal sign with a larger-than-or-equal sign
		in the definition of $R_{i, t}$:
		\begin{align} \label{eq:Rit-ge}
			R_{i, t} = \min\left\{ r \in \bN: \sum_{s = t}^{t + r} \one{C_s = i} \ge X_{i, t-1} + \one{A_t = i} \right\}.
		\end{align}
		\item We have
		\begin{align} \label{eq:Rit-eq}
			\sum_{s = t}^{t + R_{i, t}} \one{C_s = i} = X_{i, t-1} + \one{A_t = i}.
		\end{align}
		\item Lastly, we can verify that $A_t = i$ implies $C_{t + R_{i, t}} = i$.
	\end{itemize}
	
	Now focusing on $R_{i, t+1}$, we have successively:
	\begin{align*}
		R_{i, t+1}
		&\overset{\text{(a)}}{=} \min\left\{
		r \in \bN:
		\sum_{s = t+1}^{t+1+r} \one{C_s = i}
		\ge X_{i, t} + \one{A_{t+1} = i}
		\right\}, \\
		&\overset{\text{(b)}}{=} \min\left\{
		r \in \bN:
		\sum_{s = t}^{t+1+r} \one{C_s = i}
		\ge X_{i, t-1} + \one{A_t = i} + \one{A_{t+1} = i}
		\right\}, \\
		&\overset{\text{(c)}}{=} \min\left\{
		r \ge \max(0, R_{i, t} - 1):
		\sum_{s = t}^{t+1+r} \one{C_s = i}
		\ge X_{i, t-1} + \one{A_t = i} + \one{A_{t+1} = i}
		\right\}, \\
		&\overset{\text{(d)}}{=} \min\left\{
		r \ge \max(0, R_{i, t} - 1):
		\sum_{s = t + R_{i, t} + 1}^{t + 1 + r} \one{C_s = i}
		\ge \one{A_{t+1} = i},
		\right\}
	\end{align*}
	where (a) and (c) follow from~\eqref{eq:Rit-ge},
	(b) from~\eqref{eq:X-rec},
	and (d) by injecting~\eqref{eq:Rit-eq}.
	Consistently with the result announced in the lemma,
	we conclude by making a case disjunction:
	\begin{itemize}
		\item If $A_{t+1} \neq i$: The right-hand side
		of the inequality in (d) is $0$,
		hence the inequality is already satisfied
		by choosing the smallest authorized value for~$R_{i, t}$.
		\item If $A_{t+1} = i$: The right-hand side
		of the inequality in (d) is~$1$.
		Since the sum $\sum_{s = t + R_{i, t} + 1}^{t + 1 + r} \one{C_s = i}$
		is~$0$ (empty) when $r = R_{i, t} - 1$,
		remains $0$ as long as $C_{t + 1 + r} \neq i$,
		and increases to~$1$ whenever $C_{t+r+1} = i$,
		the result follows.\qedhere
	\end{itemize}
\end{proof}

\begin{corollary} \label{coro:Dit-bijection}
	For each $i \in \{1, 2, \ldots, n\}$,
	the function $t \in \bN_{> 0} \mapsto t + D_{i, t}$
	defines a bijective increasing mapping
	from the set $\{t \in \bN_{> 0}: A_t = i\}$
	into the set $\{t \in \bN_{> 0}: C_t = i\}$.
\end{corollary}

\begin{proof}
	Let $i \in \{1, 2, \ldots, n\}$.
	By definition of $D_{i, t}$,
	we have $D_{i, t} = R_{i, t}$ for each $t \in \bN_{> 0}$ so that $A_t = i$,
	hence it suffices to prove the result for $R_{i, t}$ instead of $D_{i, t}$.
	
	For each $t \in \bN_{> 0}$,
	$t + R_{i, t}$ is the round at the end of which
	client~$i$ finishes processing
	the last task that has arrived at this client
	by the beginning of round~$t$.
	\Cref{lem:Rit} implies that,
	for each $t \in \bN_{> 0}$, we have
	\begin{align*}
		t + 1 + R_{i, t+1}
		&= \begin{cases}
			t + \max(1, R_{i, t})
			&\text{if $A_{t+1} \neq i$}, \\
			\min\{s > t + R_{i, t}: C_s = i\}
			&\text{if $A_{t+1} = i$.}
		\end{cases}
	\end{align*}
	It follows directly that the function $t \in \bN_{> 0} \mapsto t + R_{i, t}$
	is non-decreasing
	and defines an injection
	from $\{t \in \bN_{> 0}: A_t = i\}$
	onto $\{t \in \bN_{> 0}: C_t = i\}$.
	Furthermore,
	to prove this injection is also a surjection,
	it suffices to verify that
	$A_{t+1} \neq i$ implies
	either $t + 1 + R_{i, t+1} = t + R_{i, t}$
	or $C_{t + 1 + R_{i, t+1}} \neq i$.
	If $A_{t+1} \neq i$,
	the only way that $t + 1 + R_{i, t+1} > t + R_{i, t}$
	is when $R_{i, t} = 0$,
	so that the maximum is attained at~$1$
	and $t + 1 + R_{i, t+1} = t + 1$.
	Using~\eqref{eq:Rit}, we can verify that
	$R_{i, t} = 0$ implies $X_{i, t} = 0$
	which, combined with $A_{t+1} \neq i$, in turn implies
	$C_{t+1} (= C_{t + 1 + R_{i, t+1}}) \neq i$ by definition of $C_{t+1}$.
\end{proof}

\subsection[Proof of Equation~\ref{eq:D}]{Proof of Equation~\eqref{eq:D}} \label{app:D}

Let $i \in \{1, 2, \ldots n\}$.
Our goal in this section is to prove that
$\bE[D_i] = \bE[X_i]$.
In a nutshell, we will prove that
\begin{align*}
	\bE[D_i]
	&\overset{(a)}{=} \lim_{T \to +\infty} \frac1T \sum_{t = 1}^T D_{i, t}
	\overset{(b)}{=} \lim_{T \to +\infty} \frac1T \sum_{t = 1}^T X_{i, t}
	\overset{(c)}{=} \bE[X_i],
\end{align*}
where all equal signs hold almost surely.
Equation~(a) will be shown in~\Cref{lem:limit}
using an ergodicity argument.
Equation~(b) will be proved using a squeeze argument
formulated in \Cref{lem:squeeze,lem:limRit}.
Lastly, Equation~(c) follows from the classical ergodic theorem
for irreducible positive-recurrent Markov chains.
These arguments are combined at the end of the section to prove~\eqref{eq:D}.

\begin{lemma} \label{lem:limit}
	For each $i \in \{1, 2, \ldots, n\}$, we have
	\begin{align*}
		\lim_{T \to +\infty} \frac1T \sum_{t = 1}^T D_{i, t} = \bE[D_i]
		\quad \text{almost surely}.
	\end{align*}
\end{lemma}

\begin{proof}[Proof of \Cref{lem:limit}]
	Equations~\eqref{eq:Dit} and~\eqref{eq:Rit}
	show that, for each $i \in \{1, 2, \ldots, n\}$,
	there exists a deterministic function
	$g_i: (\bN \times \{1, 2, \ldots, n\} \times \{1, 2, \ldots, n\})^\bN \to \bN$
	such that we can write
	$D_{i, t} = g_i((X_{t+s}, A_{t+s}, C_{t+s}), s \in \bN)$
	for each $t \in \bN$.
	Since the sequence $((X_t, A_t, C_t), t \in \bN)$ is ergodic,
	the conclusion follows from
	\citep[Remark~16.1.11]{bremaud2020}.
\end{proof}

\begin{lemma} \label{lem:squeeze}
	Let $i \in \{1, 2, \ldots, n\}$.
	If $X_{i, 0} = 0$, then for each $T \in \{1, 2, 3, \ldots\}$, we have
	\begin{align} \label{eq:bounds}
		\sum_{t = 1}^T X_{i, t}
		&\le \sum_{t = 1}^T D_{i, t}
		\le \sum_{t = 1}^{T + R_{i, T}} X_{i, t}.
	\end{align}
\end{lemma}

\begin{proof}[Proof of \Cref{lem:squeeze}]
	Let $i \in \{1, 2, \ldots, n\}$
	and $T \in \bN_{> 0}$,
	and assume that $X_{i, 0} = 0$.
	If the set $\{t \in \{1, 2, \ldots, T\}: A_t = i\}$ is empty,
	then all three sums are zero, and the inequalities are trivially satisfied. Therefore, in the remainder, we focus on sample paths for which this set is nonempty.
	
	First observe that, for each $t \in \bN$, we have
	\begin{align} \label{eq:X-unfolded}
		X_{i, s} = \sum_{t = 1}^{+\infty} \one{A_t = i, t \le s < t + D_{i, t}}.
	\end{align}
	Indeed, we have successively:
	\begin{align*}
		X_{i, s}
		&\overset{\textrm{(a)}}{=}
		\sum_{t = 1}^s \one{A_t = i}
		- \sum_{t = 1}^s \one{C_t = i}
		\overset{\textrm{(b)}}{=}
		\sum_{t = 1}^s \one{A_t = i}
		- \sum_{t = 1}^s \one{A_t = i, t + D_{i, t} \le s},
	\end{align*}
	where (a) follows by unfolding~\eqref{eq:X-rec}
	and using our assumption that $X_{i, 0} = 0$,
	and (b) follows by observing that,
	by~\Cref{coro:Dit-bijection},
	the function $t \mapsto t + D_{i, t}$
	defines a bijective (increasing) mapping
	from the set $\{t \in \bN_{> 0}: A_t = i\}$
	onto the set $\{t \in \bN_{> 0}: C_t = i\}$.
	Equation~\eqref{eq:X-unfolded} then follows by rearranging the terms.
	
	Now, we can also use the following trick
	to rewrite $\sum_{t = 1}^T D_{i, t}$
	in a manner that is similar to~\eqref{eq:X-unfolded}:
	\begin{align}
		\nonumber
		\sum_{t = 1}^T D_{i, t}
		&= \sum_{t = 1}^T
		\one{A_t = i} D_{i, t}
		\nonumber
		= \sum_{t = 1}^T
		\one{A_t = i}
		\sum_{s = 1}^{+\infty}
		\one{t \le s < t + D_{i, t}}, \\
		\label{eq:sumD}
		&= \sum_{s = 1}^{+\infty}
		\sum_{t = 1}^T
		\one{A_t = i, t \le s < t + D_{i, t}}.
	\end{align}
	The lower bound in~\eqref{eq:bounds} follows by
	capping the outer sum at $s \in \{1, 2, \ldots, T\}$ in~\eqref{eq:sumD}
	and injecting \eqref{eq:X-unfolded}.
	The upper-bound in~\eqref{eq:bounds} follows in a similar spirit:
	\begin{align*}
		\sum_{t = 1}^T D_{i, t}
		&\overset{\text{(a)}}{=}
		\sum_{s = 1}^{T + R_{i, T}} \sum_{t = 1}^T
		\one{A_t = i, t \le s < t + D_{i, t}}
		\overset{\text{(b)}}{\le}
		\sum_{s = 1}^{T + R_{i, T}} X_{i, s},
	\end{align*}
	where (a) is equivalent to~\eqref{eq:sumD},
	after recalling that tasks leave clients
	in the same order as they arrive,
	and (b) follows by expanding the inner sum to $t \in \bN_{> 0}$.
\end{proof}

\begin{lemma} \label{lem:limRit}
	For each $i \in \{1, 2, \ldots, n\}$,
	we have almost surely that
	$R_{i, T} = o(T)$ as $T \to +\infty$.
\end{lemma}

\begin{proof}[Proof of \Cref{lem:limRit}]
Let $i \in \{1, 2, \ldots, n\}$.  
For each $t \in \bN$, we have $0 \le R_{i, t} \le E_{i, t}$,  
where $E_{i, t} = \min\{r \in \bN: X_{i, t + r} = 0\}$  
denotes the number of rounds, starting from round~$t$, until client~$i$ first becomes empty.

Since the Markov chain $(X_t)_{t \in \bN}$ is irreducible and positive recurrent, the ergodic theorem for Markov chains \citep[Remark~16.1.11]{bremaud2020} implies that
\[
\lim_{t \to +\infty} \frac{1}{t} \sum_{k=1}^t E_{i, k} = \mathbb{E}[E_i],
\text{  almost surely.}
\]
Where $\mathbb{E}[E_i]$ is the expected hitting time, under the stationary distribution of $(X_t)_{t \in \bN}$, to the set of states where client~$i$ is empty.

This yields
\[
\lim_{t \to +\infty} \frac{E_{i, t}}{t}
= \lim_{t \to +\infty} \left( \frac{1}{t} \sum_{k=1}^t E_{i, k} - \frac{1}{t} \sum_{k=1}^{t-1} E_{i, k} \right)
= \mathbb{E}[E_i] - \mathbb{E}[E_i] = 0.
\]
Since $0 \le \frac{R_{i, t}}{t} \le \frac{E_{i, t}}{t}$, we conclude that $\lim_{t \to +\infty} \frac{R_{i, t}}{t} = 0$,  almost surely.
\end{proof}

\begin{proof}[Proof of Equation~\eqref{eq:D}]
	Let $i \in \{1, 2, \ldots, n\}$.
	Let us assume for now that
	the initial distribution of the Markov chain $(X_t, t \in \bN)$
	is such that $X_{i, 0} = 0$ (with probability~$1$).

	Let us first prove that the upper and lower bound in \Cref{lem:squeeze}
	converge almost surely to the same limit,
	and that this limit is~$\bE[X_i]$, that is:
	\begin{align} \label{eq:squeeze}
		\lim_{T \to +\infty}
		\frac{1}{T} \sum_{t=1}^{T} X_{i, t} 
		&\overset{(a)}{=} \bE[X_i]
		\overset{(b)}{=} \lim_{T \to +\infty}
		\frac1T \sum_{t = 1}^{T + R_{i, T}} X_{i, t}.
	\end{align}
	Since $(X_t, t \in \bN)$ is an irreducible positive-recurrent Markov chain,
	(a) follows directly from
	the ergodic theorem \citep[Theorem 3.3.2]{bremaud2020}.
	The argument for~(b) is in a similar spirit,
	with the extra-complication that the upper bound of summation ($T + R_{i, T}$)
	is different from the denominator ($T$).
	We start by rewriting the upper bound as follows:
	\begin{align*}
		\frac1T \sum_{t = 1}^{T + R_{i, T}} X_{i, t}
		&= \left( 1 + \frac{R_{i, T}}T \right)
		\frac1{T + R_{i, T}}
		\sum_{t = 1}^{T + R_{i, T}} X_{i, t}.
	\end{align*}
	Equation~(b) then follows by combining two arguments:
	(i) \Cref{lem:limRit} implies that
	$\lim_{T \to +\infty} 1 + \frac{R_{i, T}}T = 1$ almost surely;
	(ii) since $\lim_{T \to +\infty} T + R_{i, T} = +\infty$,
	the ergodic theorem for irreducible positive-recurrent Markov chains again implies that
	\begin{align*}
		\lim_{T \to +\infty}
		\frac1{T + R_{i, T}}
		\sum_{t = 1}^{T + R_{i, T}} X_{i, t}
		= \bE[X_i],
		\quad \text{almost surely.}
	\end{align*}
	
	Now, combining~\eqref{eq:squeeze} with \Cref{lem:squeeze} and the squeeze theorem
	allows us to conclude that
	\begin{align*}
		\lim_{T \to +\infty} \frac1T \sum_{t = 1}^T D_{i, t} = \bE[X_i],
		\quad \text{almost surely}.
	\end{align*}
	Equation~\eqref{eq:D} then follows from \Cref{lem:limit}.
	
	To prove that~\eqref{eq:D} also holds
	without the assumption that $X_{i, 0} = 0$ with probability~$1$,
	it suffices to recall that
	the expectations $\bE[D_i]$ and $\bE[X_i]$
	do not depend on the initial distribution.
\end{proof}

\subsection[Proof of Equation~\ref{eq:gradD}]{Proof of Equation~\eqref{eq:gradD}} \label{app:gradD}

Let $i, j \in \{1, 2, \ldots, n\}$.
Our goal is to prove~\eqref{eq:gradD},
which by~\eqref{eq:D} is equivalent to
\begin{align*}
	\frac{\partial \bE[X_i]}{\partial (\log p_j)}
	&= \cov[X_i, X_j].
\end{align*}
Recall that the vector $X$ follows the stationary distribution~\eqref{eq:pi},
which we can rewrite as
\begin{align} \label{eq:pi-log}
	\log \bP(X = x)
	&= \log \pi_{n, m-1}(x)
	= \sum_{i = 1}^n x_i \left( \log p_i - \log \mu_i \right) - \log Z_{n, m-1},
	\quad x \in \cX_{n, m-1},
\end{align}
where $Z_{n, m-1}$ follows by normalization:
\begin{align} \label{eq:C-log}
	\log Z_{n, m-1}
	&= \log \left(
	\sum_{x \in \cX_{n, m-1}}
	\exp \left( \sum_{i = 1}^n x_i \left( \log p_i - \log \mu_i \right) \right)
	\right).
\end{align}

Let us first prove the following intermediary result:
\begin{align} \label{eq:grad-log}
	\frac{\partial \log Z_{n, m-1}}{\partial (\log p_j)}
	&= \mathbb{E}[X_j],
	&
	\frac{\partial \log \pi_{n, m-1}(x)}{\partial (\log p_j)}
	&= x_j - \mathbb{E}[X_j],
	\quad x \in \cX_{n, m-1}.
\end{align}

The first part of~\eqref{eq:grad-log} follows by
taking the partial derivative of~\eqref{eq:C-log}
and rearranging the terms to retrieve the definition of~$\pi_{n, m-1}$:
\begin{align*}
	\frac{\partial \log(Z_{n, m-1})}{\partial (\log p_j)}
	&= \frac1{Z_{n, m-1}} \frac{\partial Z_{n, m-1}}{\partial (\log p_j)}
	= \frac1{Z_{n, m-1}}
	\sum_{x \in \cX_{n, m-1}}
	x_j
	\exp \left( \sum_{i = 1}^n x_i (\log p_i - \log \mu_i) \right), \\
	&= \sum_{x \in \cX_{n, m-1}}
	x_j \exp \left( \sum_{i = 1}^n x_i (\log p_i - \log \mu_i) - \log Z_{n, m-1} \right)
	= \sum_{x \in \cX_{n, m-1}} x_j \pi_{n, m-1}(x)
	= \mathbb{E}[X_j].
\end{align*}
Now, the second part of~\eqref{eq:grad-log} follows
by taking the partial derivative of~\eqref{eq:pi-log}
and injecting the previous result:
\begin{align*}
	\frac{\partial \log \pi_{n, m-1}(x)}{\partial (\log p_j)}
	&= x_j - \frac{\partial \log Z_{n, m-1}}{\partial (\log p_j)}
	= x_j - \bE[X_j],
	\quad x \in \cX_{n, m-1}.
\end{align*}

To conclude, it suffices to inject the second part of~\eqref{eq:grad-log}
into the definition of expectation:
\begin{align*}
	\frac{\partial \mathbb{E}[X_i]}{\partial (\log p_j)}
	&= \sum_{x \in \cX_{n, m-1}}
	x_i \frac{\partial \pi_{n, m-1}(x)}{\partial (\log p_j)}
	= \sum_{x \in \cX_{n, m-1}}
	\pi_{n, m-1}(x) x_i \frac{\partial \log \pi_{n, m-1}(x)}{\partial (\log p_j)}, \\
	&= \sum_{x \in \cX_{n, m-1}}
	\pi_{n, m-1}(x) x_i (x_j - \mathbb{E}[X_j])
	= \cov[X_i, X_j].
\end{align*}

\subsection[Proof of Equations~\ref{eq:Xi} and~\ref{eq:XiXj}]{Proof of Equations~\eqref{eq:Xi} and~\eqref{eq:XiXj}} \label{app:buzen}

Equation~\eqref{eq:Xi} was proved in \citep[Equation~(8)]{buzen:1973}.
Buzen's algorithm \cite{buzen:1973} as described in the proposition
was introduced and shown to yield the correct result
in \citep[Paragraph ``Computation of $G(N)$'']{buzen:1973}.
(Note that, in \cite{buzen:1973}, $N$ corresponds to our~$m$ and $M$ to our~$n$.)
All that remains is to prove~\eqref{eq:XiXj},
which we do in a similar way to the proof of \eqref{eq:Xi} in \cite{buzen:1973}.
To simplify notation, we can focus without loss of generality on the pair $(i, j) = (1, 2)$.
We have successively:
\begin{align*}
	\mathbb{E}[X_1 X_2]
	&= \sum_{x \in \cX_{n, m-1}} x_1 x_2 \pi_{n, m-1}(x)
	= \sum_{x \in \cX_{n, m-1}} \sum_{k = 1}^{x_1} \sum_{\ell = 1}^{x_2} \pi_{n, m-1}(x)
	= \sum_{\substack{k, \ell = 1 \\ k + \ell \le m-1}}^{m-1}
	\sum_{\substack{x \in \cX_{n, m-1} \\ x_1 \ge k, x_2 \ge \ell}}
	\pi_{n, m-1}(x), \\
	&= \frac1{Z_{n, m-1}}
	\sum_{\substack{k, \ell = 1 \\ k + \ell \le m-1}}^{m-1}
	\sum_{\substack{x \in \cX_{n, m-1} \\ x_1 \ge k, x_2 \ge \ell}}
	\prod_{i = 1}^n \left( \frac{p_i}{\mu_i} \right)^{x_i}
	\overset{(*)}= \frac1{Z_{n, m-1}}
	\sum_{\substack{k, \ell = 1 \\ k + \ell \le m-1}}^{m-1}
	\left( \frac{p_1}{\mu_1} \right)^{k}
	\left( \frac{p_2}{\mu_2} \right)^{\ell}
	\sum_{y \in \cX_{n, m - 1 - k - \ell}}
	\prod_{i = 1}^n \left( \frac{p_i}{\mu_i} \right)^{y_i}, \\
	&= \frac1{Z_{n, m-1}}
	\sum_{\substack{k, \ell = 1 \\ k + \ell \le m-1}}^{m-1}
	\left( \frac{p_1}{\mu_1} \right)^{k}
	\left( \frac{p_2}{\mu_2} \right)^{\ell}
	Z_{n, m-1-k-\ell},
\end{align*}
where ($*$) follows by making the change of variable
$y = x - k e_1 - \ell e_2$,
where $e_i$ is the $n$-dimensional vector
with one in component $i$ and zero elsewhere.

\section[Proof of Proposition~\ref{theo:bound-time}]{Proof of \Cref{theo:bound-time}} \label{app:time}

\paragraph{Proof of \Cref{eq:throughput}}

We start by proving \Cref{eq:throughput}. Let the sequence $\{t_k, k \in \mathbb{N}\}$ denote the wall-clock time instants at which round~$k$ starts, and let $\{\xi_i(t), t \in \mathbb{R}_{\ge 0}\}$ represent the number of tasks at client~$i$ at (wall-clock) time~$t$, for all $i \in \{1, \ldots, n\}$. The throughput $\lambda$ is defined as the average number of rounds completed per unit of (wall-clock) time. Mathematically, for a given time window $s \in \mathbb{R}_{>0}$, the throughput is expressed as:
\begin{align*}
    \lambda &= \mathbb{E} \left[\frac{1}{s} \sum_{k \ge 0} \one{t_k \leq s} \right]
     = \mathbb{E} \left[ \frac{1}{s} \int_0^s \sum_{i=1}^n \mu_i \one{\xi_i(t)>0} dt \right], \\
    & = \sum_{i=1}^n \mu_i \bP(\xi_i>0)
     = \sum_{i=1}^n \mu_i \frac{p_i}{\mu_i} \frac{Z_{n,m-1}}{Z_{n,m}}
     = \frac{Z_{n,m-1}}{Z_{n,m}}.
\end{align*}

The first equality follows from the definition of throughput. The second applies the stochastic intensity formula~\citep[Section 1.8.3, "Stochastic Intensity Integration Formula"]{baccelli2002} to the point process $\{t_k\}_{k \in \mathbb{N}}$, whose stochastic intensity is $\sum_{i=1}^n \mu_i \one{\xi_i(t) > 0}$. The third equality follows by interchanging expectation and integration, using the linearity of expectation and the stationarity of $\xi$. The fourth uses the expression for $\mathbb{P}(\xi_i > 0)$ from~\citep[Equation~(6)]{buzen:1973}. This completes the proof.

\paragraph{Proof of \Cref{boundH}}

Let $Y_k = \left( Y_{1,k}, Y_{2,k}, \ldots, Y_{n,k} \right)$, where $k \in \mathbb{N}$, represents the state of the network at round~$k$. The sequence $\{Y_k, k \in \mathbb{N}\}$ can be recursively defined as follows: 
\begin{itemize}
	\item $Y_0$ represents the post-jump state at $t=0$, i.e., for all $i \in \{1,\ldots,n\}$, $Y_{i, 0} = X_{i, 0} + \one{C_{0} = i}$.
	\item For each $k \in \mathbb{N}_{>0}$ and $i \in \{1,\ldots,n\}$, $Y_{i, k} = Y_{i, k-1} + \one{A_k = i} - \one{C_{k-1} = i}$.
\end{itemize}
We know that the sequence $\{\tau_k | Y_k\}_{k \in \{0, \ldots, T\}}$ consists of independent random variables, such that $\tau_k | Y_k$ is exponentially distributed with a parameter $\sum_{i=1}^n \mu_i \mathbf{1}\{Y_{i,k} \geq 1\}$, for each $k \in \{0, 1, \ldots, T\}$. Conditioning on $Y_k$, we can write for all $k \in \{0, 1, \ldots, T\}$:
\begin{align*}
    \mathbb{E}[\tau_k] &= \sum_{x \in \mathcal{X}_{n, m}} \mathbb{E}[\tau_k | Y_k=x] \mathbb{P}(Y_k=x), \\
    &= \sum_{x \in \mathcal{X}_{n, m}} \frac{1}{\sum_{i=1}^n \mu_i \mathbf{1}\{x_i \geq 1\}} \mathbb{P}(Y_k=x).
\end{align*}

The sequence $\{Y_k, k \in \mathbb{N}\}$ forms an ergodic, discrete-time, homogeneous Markov chain. Its stationary distribution can be derived by noting that it corresponds to the jump chain of the ergodic, continuous-time Markov chain $\{\xi(t), t \geq 0\}$, whose stationary distribution is $\pi_{n,m}$. Using Equation~(13.56) from Theorem~13.4.5 in \cite{bremaud2013markov}, the stationary distribution of $\{Y_k, k \in \mathbb{N}\}$ is given by:

\begin{align} \label{st:jump}
    \mathbb{P}(Y_k=x) = \frac{1}{V_{n,m}} \sum_{i=1}^n \mu_i \mathbf{1}\{x_i \geq 1\} \prod_{j=1}^n \left( \frac{p_j}{\mu_j} \right)^{x_j},
		\quad x \in \mathcal{X}_{n, m},
\end{align}
where $V_{n,m}$ is a normalizing constant.

Under the stationarity assumption (\ref{A5}) and substituting \eqref{st:jump} into the expression previously derived for $\mathbb{E}[\tau_k]$, we obtain:
\begin{align} \label{eq:E-tau_T}
    \bE[\tau_k] &= \sum_{x \in \cX_{n, m}} \frac{1}{V_{n,m}} \prod_{j = 1}^n \left( \frac{p_j}{\mu_j} \right)^{x_j} 
    \overset{(*)}{=} \frac{Z_{n,m}}{V_{n,m}} .
\end{align}

Next, we demonstrate that $V_{n,m} = Z_{n,m-1}$. By definition, we have:
\begin{align*} 
    V_{n,m} &= \sum_{x \in \cX_{n, m}} \sum_{\substack{i=1\\ x_i \geq 1 }}^n \mu_i \prod_{j = 1}^n \left( \frac{p_j}{\mu_j} \right)^{x_j}
     \overset{(a)}{=} \sum_{i=1}^n \mu_i \sum_{\substack{x \in \cX_{n, m}\\ x_i \geq 1 }} \prod_{j = 1}^n \left( \frac{p_j}{\mu_j} \right)^{x_j}
     \overset{(b)}{=} \sum_{i=1}^n p_i \sum_{\substack{x \in \cX_{n, m}\\ x_i > 0 }} \prod_{j = 1}^n \left( \frac{p_j}{\mu_j} \right)^{(x-e_i)_j},
\end{align*}
where, for each $i \in \{1,\ldots,n\}$, $e_i$ denotes the $n$-dimensional vector with one in component $i$ and zero elsewhere.
(a) is obtained by rearranging the order of summation, while (b) follows from factoring out $\frac{p_i}{\mu_i}$. Applying the variable substitution $y = x - e_i$, the expression simplifies as:
\begin{align*}
    V_{n,m} = \underbrace{\sum_{i=1}^n p_i}_{=1} \underbrace{\sum_{y \in \mathcal{X}_{n, m-1}} \prod_{j=1}^n \left( \frac{p_j}{\mu_j} \right)^{y_j}}_{=Z_{n,m-1}} = Z_{n,m-1}.
\end{align*}

Thus, incorporating this result into equality ($*$) of \eqref{eq:E-tau_T}, and using \eqref{eq:throughput}, we conclude that for all $k \in \{0, 1, \ldots, T\}$:
\begin{align} \label{eq:mean_round}
    \mathbb{E}\left[\tau_k\right] = \frac{Z_{n,m}}{Z_{n,m-1}} = \frac{1}{\lambda}.
\end{align}

Finally, we get:
\begin{align*}
    \frac{1}{T+1}\sum_{t=0}^{T} \bar{\tau}_t \bE \left[ \|\nabla f(w_t)\|^2 \right] = \frac{1}{\lambda} \frac{1}{T+1}\sum_{t=0}^{T} \bE \left[ \|\nabla f(w_t)\|^2 \right] 
    \leq \frac{1}{\lambda} 8 G
\end{align*}

\section[Proof of Proposition~\ref{prop:time-eps}]{Proof of \Cref{prop:time-eps}} \label{proof:time_eps}

As recalled in \Cref{sec:optimize-updates-delay},
\citep[Theorem~1]{leconte2024queueing} provides the following upper bound on the ergodic mean of the squared gradient norm of~$f$: there exists a constant $\eta_{\text{max}} > 0$ (which depends on $p$) such that, for any $\eta \in (0, \eta_{\text{max}})$,
\begin{align*}
    \frac{1}{T + 1}\sum_{t=0}^{T} \mathbb{E}[\|\nabla f(w_t)\|^2]
    \leq 8G, \quad \text{where} \quad
    G = \frac{A}{\eta(T + 1)}
        + \frac{\eta L B}{n^2} \sum_{i=1}^{n} \frac{1}{p_i}
        + \frac{\eta^2 L^2 B m}{n^2} \sum_{i=1}^{n} \frac{\mathbb{E}[D_i]}{p_i^2}.
\end{align*}

Assume the learning rate is a function of the number of rounds $T$, i.e., $\eta = \frac{C}{T^\alpha}$ for some constants $\alpha, C \in \mathbb{R}_{> 0}$ such that $\alpha<1$ and $\eta < \eta_{\text{max}}$. Then we can express $G$ as:
\begin{align*}
    G = \frac{A}{C T^{-\alpha} (T + 1)}
        + \frac{C L B}{n^2 T^{\alpha}} \sum_{i=1}^{n} \frac{1}{p_i}
        + \frac{C^2 L^2 B m}{n^2 T^{2\alpha}} \sum_{i=1}^{n} \frac{\mathbb{E}[D_i]}{p_i^2}.
\end{align*}
Therefore, to obtain
$$
\frac{1}{T + 1} \sum_{t=0}^T \mathbb{E}[\|\nabla f(w_t)\|^2] \leq \epsilon,
$$
it suffices to take $T \ge T_\epsilon$, where
\begin{align*}
    T_{\epsilon} = \mathcal{O}\Bigg(& 
        \left( \frac{A}{C\epsilon} \right)^{\frac{1}{1 - \alpha}}
    + \left( \frac{C L B}{\epsilon n^2} \sum_{i=1}^{n} \frac{1}{p_i} \right)^{\frac{1}{\alpha}} 
    + \left( \frac{C^2 L^2 B m}{\epsilon n^2} \sum_{i=1}^{n} \frac{\mathbb{E}[D_i]}{p_i^2} \right)^{\frac{1}{2\alpha}}
    \Bigg) \text{ rounds.}
\end{align*}
Using \Cref{eq:mean_round}, the expected wall-clock time to reach this accuracy is
$$
\mathbb{E}[\tilde{\tau}_\epsilon] = \sum_{t=0}^{T_\epsilon - 1} \mathbb{E}[\tau_t] = \frac{T_\epsilon}{\lambda}.
$$

\section[Compute G and ∇ₚ G]{Compute $G$ and $\nabla_p G$} \label{app:G}

\begin{proposition} \label{prop:bound}
In the framework of \Cref{sec:fl}, the expression of the upper bound $G(p)$ in terms of routing probabilities is given by:
\begin{align}
	G(p)
	&= \frac{A}{\eta(T + 1)}
	+ \frac{\eta L B}{n^2} \sum_{i=1}^{n} \frac{1}{p_i}
	+ \frac{\eta^2 L^2 B m}{n^2} \sum_{i=1}^{n} \sum_{k = 1}^{m-1}
	\frac{p_i^{k-2}}{\mu_i^k} 
	\frac{Z_{n, m-1-k}}{Z_{n, m-1}},
\end{align}
where the normalizing constants $Z_{n, \um}$ for $\um \in \{0, 1, \ldots, m - 1\}$ can be computed explicitly  
with $\mathcal{O}(nm)$ time and $\mathcal{O}(m)$ memory complexity, as shown in \Cref{prop:jackson}.
\end{proposition}

To identify the best routing strategy, we aim to minimize this upper bound using a gradient descent algorithm. This requires computing the gradient of the function $G(p)$ with respect to the routing probability vector, as detailed in the following proposition, which follows from \Cref{theo:little}.

\begin{proposition} \label{prop:grad}
For each $j \in \{1,\ldots,n\}$, we have
\begin{align*}
	\frac{\partial G}{\partial p_j} =
	\frac{\eta L B}{n^2 p_j}
	\left(
	-\frac{1}{p_j}
	+ \eta m L \left(
	\sum_{i=1}^n \frac{\bE[X_i X_j]}{p_i^2}
        - \bE[X_j] \sum_{i=1}^n \frac{\bE[X_i]}{p_i^2}
	- \frac{2 \bE[X_j]}{p_j^2}
	\right)
	\right),
\end{align*}
where the random vector $X \in \cX_{n, m-1}$ is distributed according to $\pi_{n, m-1}$, as defined in Equation~\eqref{eq:pi}. Moreover, for all $(i, j) \in \{1, \ldots, n\}^2$, the quantities $\bE[X_j]$ and $\bE[X_i X_j]$ can be computed using the formulas from \Cref{theo:little}, where the normalizing constants are efficiently computed in $\mathcal{O}(nm)$ time and $\mathcal{O}(m)$ memory complexity using Buzen's algorithm, as described in \Cref{prop:jackson}.
\end{proposition}

To enforce the probability constraints on \( p = (p_i)_{i=1}^n \), we introduce auxiliary parameters \( \Theta = (\theta_i)_{i=1}^n \) and apply the softmax transformation:
\begin{align*}
	p_j = \frac{\re^{\theta_j}}{\sum_{i=1}^n \re^{\theta_i}}, \quad j \in \{1, \ldots, n\}.
\end{align*}
This guarantees that the resulting \( p \) forms a valid probability vector.

The gradient of the upper bound~$G$ with respect to \( \Theta \) is given by:
\begin{align*}
	\frac{\partial G}{\partial \theta_j}
	= \left\langle \nabla_p G , \frac{\partial p}{\partial \theta_j} \right\rangle, \quad 
	\text{where } \frac{\partial p}{\partial \theta_j} = p_j (e_j - p), \text{ for each \( j \in \{1, \ldots, n\} \).}
\end{align*}
Here, \( e_j \) denotes the $n$-dimensional unit vector with $1$ in coordinate $j$, and \( \langle \cdot , \cdot \rangle \) denotes the dot product in \( \mathbb{R}^n \).

\section[Compute H and ∇ₚ H]{Compute $H$ and $\nabla_p H$} \label{app:H}

\begin{proposition} \label{prop:boundH}
	In the framework of \Cref{sec:fl}, the expression of the upper bound $H(p)$ in terms of routing probabilities is given by:
	\begin{align}
		H(p)
		&= \frac{Z_{n,m}}{Z_{n,m-1}} \left( \frac{A}{\eta(T+1)}
		+ \frac{\eta L B}{ n^2} \sum_{i=1}^{n} \frac{1}{p_i} 
		+ \frac{\eta^2 L^2 B m}{n^2} \sum_{i=1}^{n} \sum_{k = 1}^{m-1}
		\frac{p_i^{k-2}}{\mu_i^k} 
		\frac{Z_{n, m-1-k}}{Z_{n, m-1}} \right),
	\end{align}
where the normalizing constants $Z_{n, \um}$ for $\um \in \{0, 1, \ldots, m\}$ can be computed explicitly  
with $\mathcal{O}(nm)$ time and $\mathcal{O}(m)$ memory complexity, as shown in \Cref{prop:jackson}.
\end{proposition}	

To find the optimal routing strategy, we aim to minimize this upper bound using a gradient descent algorithm. This requires the gradient of $H(p)$ with respect to the routing probabilities, given in the following proposition.

\begin{proposition} \label{prop:gradH}
For each $j \in \{1,\ldots,n\}$, the gradient $\nabla_p H$ with respect to the routing probability $p_j$ is given by:
\begin{align*}
	\frac{\partial H}{\partial p_j} =
        \frac{A \bE[\xi_j - X_j]}{(T+1)\eta \lambda p_j} +
	\frac{\eta L B}{n^2 p_j \lambda}
	\left(
        -\frac{1}{p_j} +
        \bE[\xi_j - X_j] \sum_{i=1}^n \frac{1}{p_i}
	+ \eta m L \left(
        - \frac{2 \bE[X_j]}{p_j^2} 
        + \sum_{i=1}^n \frac{\bE[X_i X_j]}{p_i^2}
        + \bE[\xi_j - 2 X_j] \sum_{i=1}^n \frac{\bE[X_i]}{p_i^2}
	\right)
	\right),
\end{align*}
where the random vectors $\xi \in \cX_{n, m}$ and $X \in \cX_{n, m-1}$ are distributed according to $\pi_{n, m}$ and $\pi_{n, m-1}$, respectively, as defined in Equation~\eqref{eq:pi}. 

Moreover, for all $(i, j) \in \{1, \ldots, n\}^2$, the quantities $\bE[X_j]$, $\bE[\xi_j]$, and $\bE[X_i X_j]$ can be computed using the formulas from \Cref{theo:little}, where the normalizing constants are efficiently computed in $\mathcal{O}(nm)$ time and $\mathcal{O}(m)$ memory complexity using Buzen's algorithm, as described in \Cref{prop:jackson}.
\end{proposition}

The proxy objective~$H$ is optimized using the same softmax reparameterization as in Section~\ref{app:G}.

\section{Experiments Details} \label{exp:dl}

\subsection{Neural Networks Architectures}
The Fashion-MNIST and KMNIST datasets each contain 70,000 grayscale images, with 60,000 designated for training and 10,000 reserved for testing. These images, sized at 28×28 pixels, are evenly distributed across 10 classes.
On the other hand, the CIFAR-10 and CIFAR-100 datasets feature 60,000 color images (RGB) with a resolution of 32×32 pixels. Of these, 50,000 are used for training, while 10,000 are allocated for testing. CIFAR-10 is organized into 10 classes, while CIFAR-100 expands to 100 classes.
Notably, all these datasets are class-balanced, ensuring each class has an equal number of images.

For the Fashion-MNIST and KMNIST datasets, we employ a convolutional neural network (CNN) with the following structure:
\begin{itemize}
	\item Two convolutional layers with 7x7 filters, each followed by a ReLU activation function. The first convolutional layer has 20 channels, while the second has 40 channels.
	\item A 2x2 max pooling layer.
	\item A final fully connected layer with 10 neurons, concluded by a softmax activation function.
\end{itemize}
For the CIFAR-10 and CIFAR-100 datasets, we employ a convolutional neural network (CNN) with the following architecture:
\begin{itemize}
    \item \textbf{Three sequential convolutional blocks}, each consisting of two 3$\times$3 convolutional layers. Each layer is followed by ReLU activation and Group Normalization. The configuration of these blocks is as follows:
    \begin{itemize}
        \item The first block contains convolutional layers with 32 channels.
        \item The second block contains convolutional layers with 64 channels.
        \item The third block contains convolutional layers with 128 channels.
    \end{itemize}
    Additionally, each block includes 2$\times$2 max pooling and a dropout layer with a probability of 0.25.\\

    \item \textbf{A classification block} comprising:
    \begin{itemize}
        \item A flattening layer.
        \item A fully connected layer with 128 neurons.
        \item A dropout layer with a probability of 0.25.
        \item A final fully connected layer with the number of neurons corresponding to the number of classes (10 for CIFAR-10 and 100 for CIFAR-100), followed by a softmax activation function.
    \end{itemize}
\end{itemize}

In all experiments, the stochastic gradient for each task is computed using a batch size of 512 data points. The numerical implementation is carried out in PyTorch, and the experiments are performed on an NVIDIA Tesla P100 GPU.

\subsection{Closed Jackson Network Simulation}

For each $t \in \{0, \ldots, T\}$, recall that $Y_t$ denote the $n$-dimensional random variable representing the queue lengths during round~$t$, and $\tau_t$ represent the duration (in  wall-clock time) of round~$t$.
The Jackson network simulation proceeds by initializing $Y_0$, then iterating through the following steps for each round~$t \in \{0, \ldots, T\}$:
\begin{itemize}
    \item Sample $\tau_t$ from an exponential distribution with rate parameter $\sum_{j=1}^n \mu_j \mathbf{1}\{Y_{j,t} > 0\}$, where $\mu_j$ represents the processing speed of client $j$ and $Y_{j,t}$ denotes the queue length of client $j$ at round~$t$.
    
    \item Select a client $k$ to complete a task at the end of round~$t$ from the set of non-empty queues. The probability of selecting client $k$ is proportional to the processing speeds of the clients, given by the distribution:
    \[
    \left( \frac{\mu_i \mathbf{1}\{Y_{i,t} > 0\}}{\sum_{j=1}^n \mu_j \mathbf{1}\{Y_{j,t} > 0\}}, i \in \{1, \ldots, n\} \right).
    \]

    \item Reassign the task to another client $l$ based on the routing distribution $(p_i, i \in \{1, \ldots, n\})$. The queue lengths are then updated as follows:
    \[
    Y_{t+1} = Y_t - e_k + e_l,
    \]
    where $e_k$ and $e_l$ are unit vectors indicating that the task is removed from client $k$ and added to client $l$, respectively. This concludes round~$t$ and initiates the next round.
\end{itemize}

\section{Additional Experiments} \label{app:num-additional}

\subsection{Additional Experiments for Section~\ref{num:optimize-updates}}

To evaluate the robustness of the optimized routing strategy~$p^*_G$ with respect to the proxy objective~$G$ under challenging heterogeneous data distributions, and to assess its sensitivity to different computation time distributions, we conducted additional experiments on the Fashion-MNIST, CIFAR-10, and CIFAR-100 datasets.

\subsubsection{Robustness to Data Distribution Heterogeneity} \label{app:num-dist-heterogeneity}

\paragraph{Highly heterogeneous data distribution}

This scenario follows the setup of \Cref{num:optimize-updates}, but with a different data partitioning across clients. For Fashion-MNIST and CIFAR-10, each client was assigned 3 unique image labels out of 10, distributed sequentially and cyclically. For example, the first client received labels 0, 1, and 2; the second, labels 3, 4, and 5; and so on, wrapping around after label 9. The same procedure was applied to CIFAR-100, but with 5 labels per client.

This setup creates a non-\gls{iid} data distribution, as each client can only access a small subset of labels. All clients were given an equal number of images.

We compare the test accuracy and loss trajectories of the optimized, uniform, and balanced routing strategies. To eliminate the impact of initialization, all neural networks were initialized with the same weights across routing strategies. For each dataset and strategy, we performed multiple independent simulations using different random seeds and recorded test performance.

The results in \Cref{fig:sim_3classes} show that the optimized routing strategy ultimately achieves the highest accuracy, with a notably faster performance gain compared to the other methods. In contrast, the uniform and balanced routing strategies exhibit greater volatility across independent simulations, along with more frequent loss spikes.

\begin{figure*}[hbtp]
	\centering
    \includegraphics[width=0.7\textwidth]{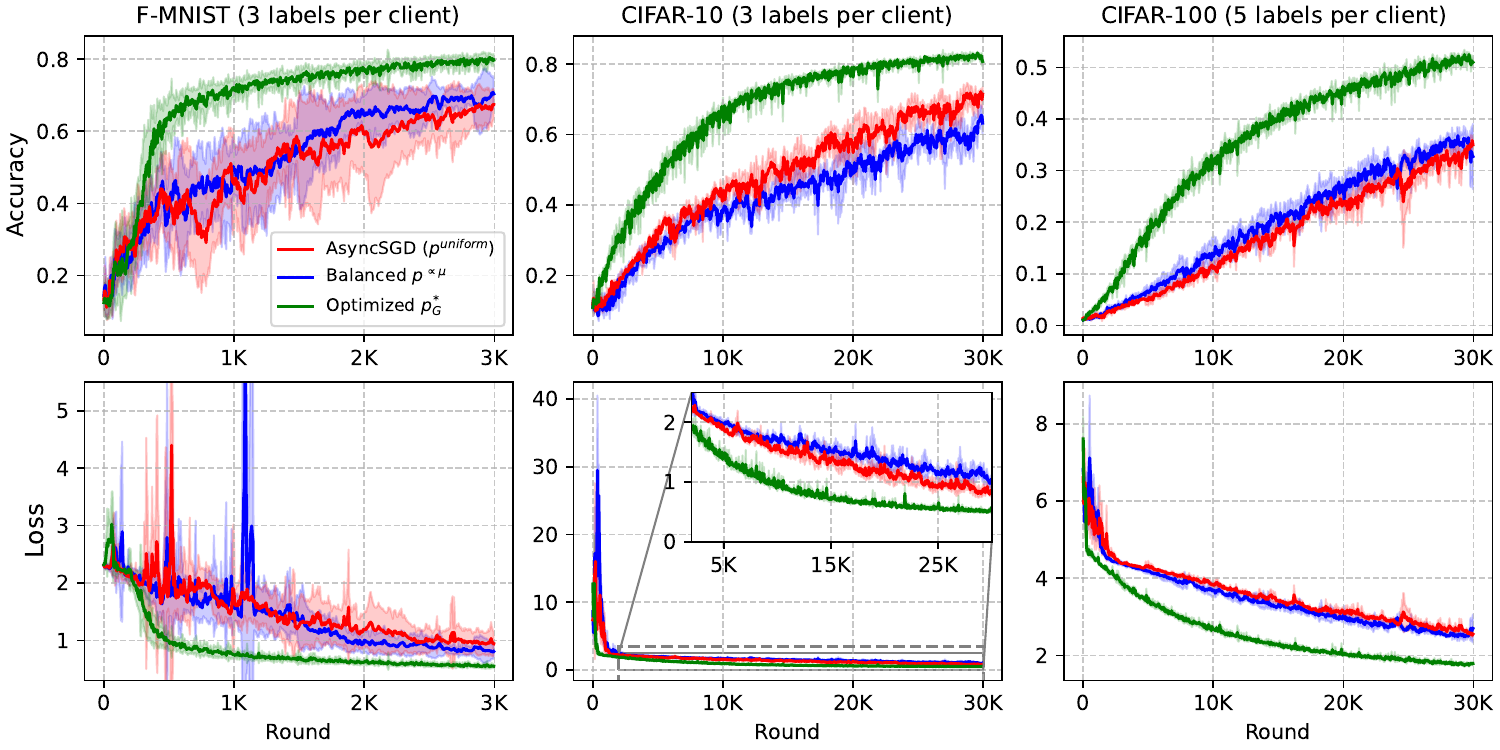} 
    \caption{Performance on the test set at the \gls{CS} in the scenario of \Cref{num:optimize-updates}, with $n = 20$ clients and $m = 100$ tasks under highly heterogeneous data splits. For Fashion-MNIST, we simulated training over 3{,}000 rounds, repeated 10 times, recording accuracy and loss every 5 rounds. For CIFAR-10 and CIFAR-100, we applied standard normalization and data augmentation, ran each simulation for 30{,}000 rounds, and repeated it three times, logging metrics every 50 rounds on an unseen test set. Solid lines show metrics averaged over independent runs; shaded areas represent standard deviations.}
	\label{fig:sim_3classes}
\end{figure*}

\paragraph{Disjoint datasets}

To further validate these findings, we examined another challenging data heterogeneous setting involving a system with 10 clients and 100 tasks. In this scenario, each image label from the 10 labels in the Fashion-MNIST dataset is assigned exclusively to a single client, ensuring that each client has a completely distinct data distribution. Additionally, the service speed of each client~$i$ is defined as $\mu_i = \textrm{e}^{i/50}$, resulting in the fastest client being approximately 20\% faster than the slowest. The learning rate is set to $\eta = 0.005$, and the smoothness constant is $L = 1$.

In \Cref{fig:sim_1class}, we compare the evolution of accuracy and loss under the optimized routing strategy against the uniform and balanced routing strategies in this more complex environment. Similar to the previous experiment, training spans 3,000 rounds, with metrics recorded on a test dataset at intervals of 5 rounds.

Our results further validate that the optimized routing strategy~$p^*_G$ consistently surpasses the alternatives $p^{\text{uniform}}$ and $p^{\propto \mu}$, achieving superior accuracy with a consistently faster performance gain. Furthermore, it demonstrates greater robustness and stability in this highly heterogeneous environment, showing a smaller standard deviation in both loss and accuracy across the 10 independent simulations, along with fewer loss spikes compared to the uniform and balanced strategies. Additionally, although the optimized routing often selects the slowest client more frequently, the learning process does not develop a bias toward optimizing the local loss of this client. Instead, it preserves balanced global performance, effectively avoiding overfitting to the image labels associated with any specific client.

\begin{figure*}[hbtp]
	\centering
	\includegraphics[width=0.7\textwidth]{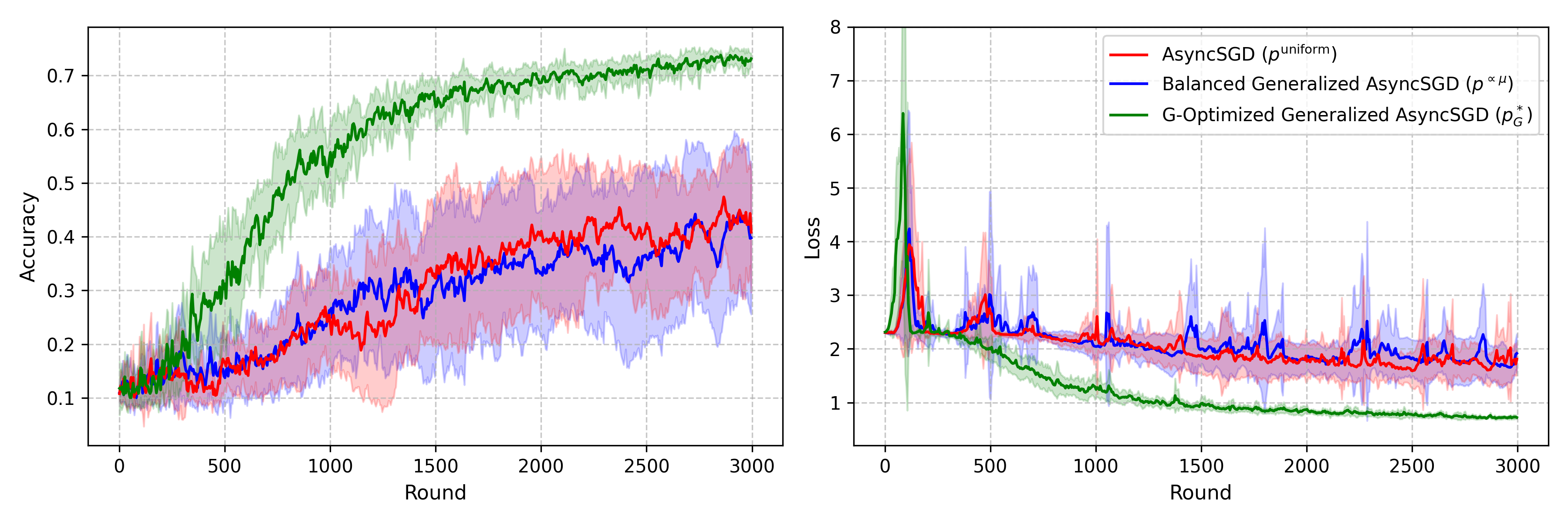} 
	\caption{Performance on the test set at the \gls{CS} with 10 clients and 100 tasks under heterogeneous data splits. Solid lines show metrics averaged over independent simulations; shaded areas represent standard deviations.}
	\label{fig:sim_1class}
\end{figure*}

\subsubsection{Robustness to Computation Time Distributions} \label{app:num-compu-time-heterogeneity}

To evaluate the robustness of the proposed approach under different computation time distributions, we retain the setup from \Cref{num:optimize-updates} but modify the distribution, while maintaining the same average task completion rate at each client.

We study two scenarios: (i) deterministic computation times, where computations at each client~$i$ have a fixed duration equal to $\frac{1}{\mu_i}$; and (ii) lognormal computation times, a heavy-tailed distribution used to model realistic processing delays in distributed systems. In the latter case, the mean is set to $\frac{1}{\mu_i}$ for each client~$i$, and the standard deviation of the underlying normal distribution is fixed at $\sigma_s = 1$, ensuring a constant coefficient of variation across clients. For all $i \in \{1,\ldots,n\}$, the values of $\mu_i$ match those used in the exponential setting of \Cref{num:optimize-updates}.

We compare the optimized routing strategy~$p^*_G$ (computed as in \Cref{num:optimize-updates}) with the uniform and balanced strategies under both homogeneous and heterogeneous data settings. Heterogeneity is introduced using Dirichlet-based sampling ($\text{Dir}_n(0.5)$) or disjoint label partitions across clients. For Fashion-MNIST, training was simulated for 3{,}000 rounds and repeated 10 times, with accuracy and loss recorded every 5 rounds. For CIFAR-10 and CIFAR-100, we applied standard normalization and data augmentation, ran each simulation for 30{,}000 rounds, repeated three times, and logged performance every 50 rounds on an unseen test set. All neural networks were initialized with the same weights across routing strategies to ensure fair comparison.

The results are shown in \Cref{fig:sim_deter,fig:sim_lognormal}, where solid lines indicate averages over independent runs and shaded regions represent standard deviations. They indicate that performance is largely insensitive to the choice of computation time distribution, with trends closely matching those under the exponential assumption. The $G$-optimized routing consistently achieves the highest accuracy, with faster performance gain and lower volatility in both accuracy and loss trajectories.

\begin{figure}[htbp]
  \centering
    \includegraphics[width=0.49\textwidth]{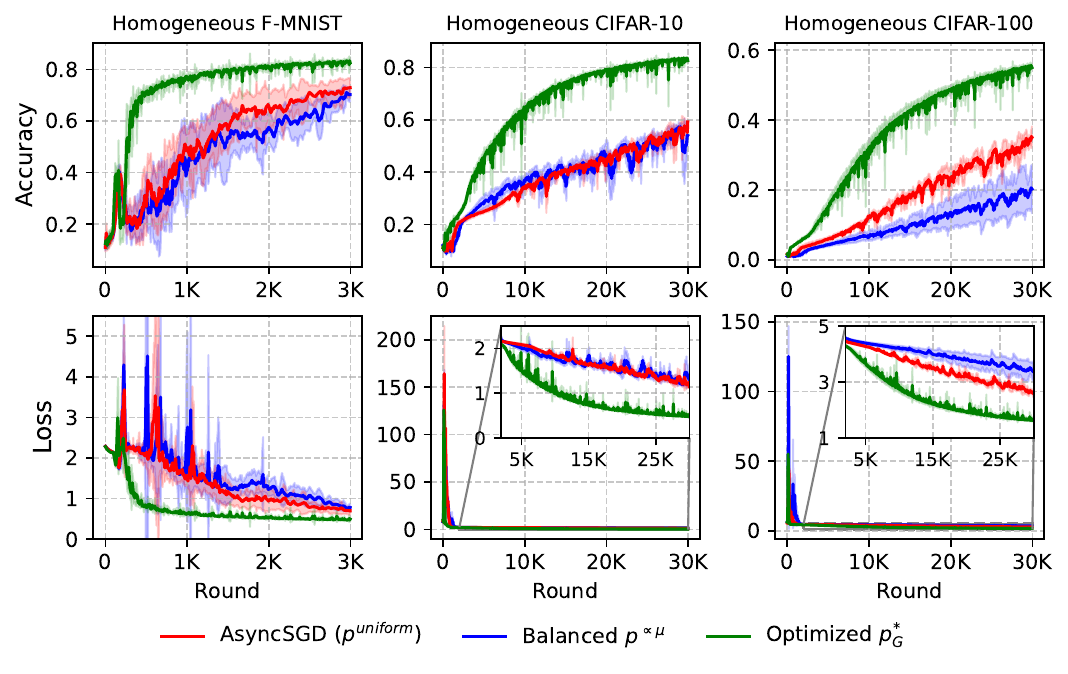}
  \hfill
    \includegraphics[width=0.49\textwidth]{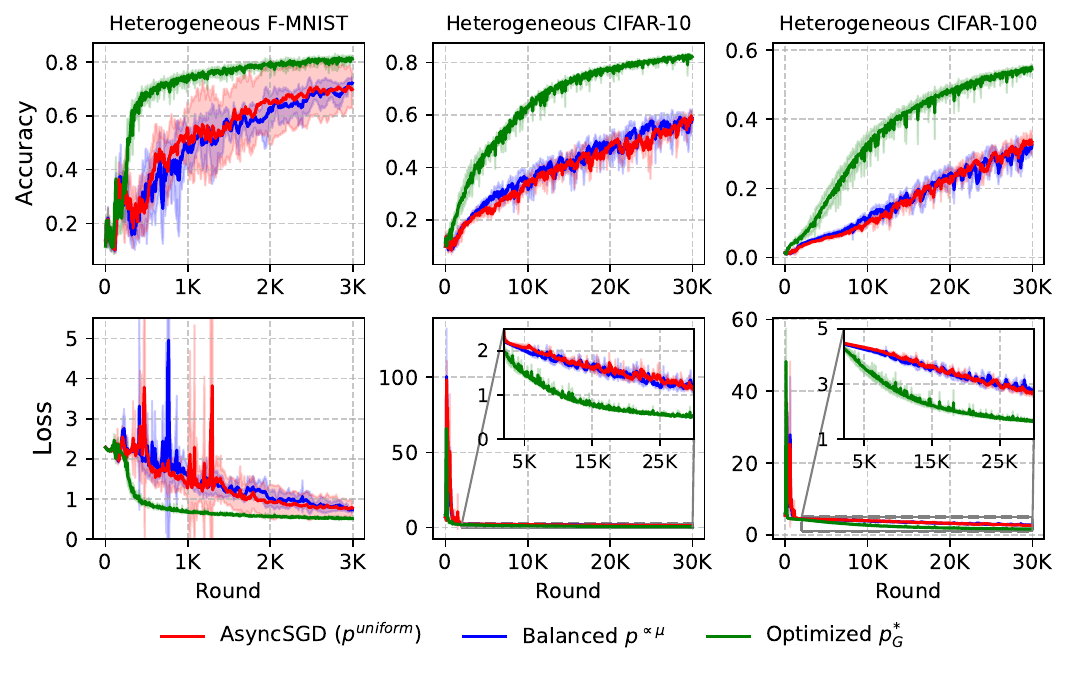}
  \\
    \includegraphics[width=0.49\textwidth]{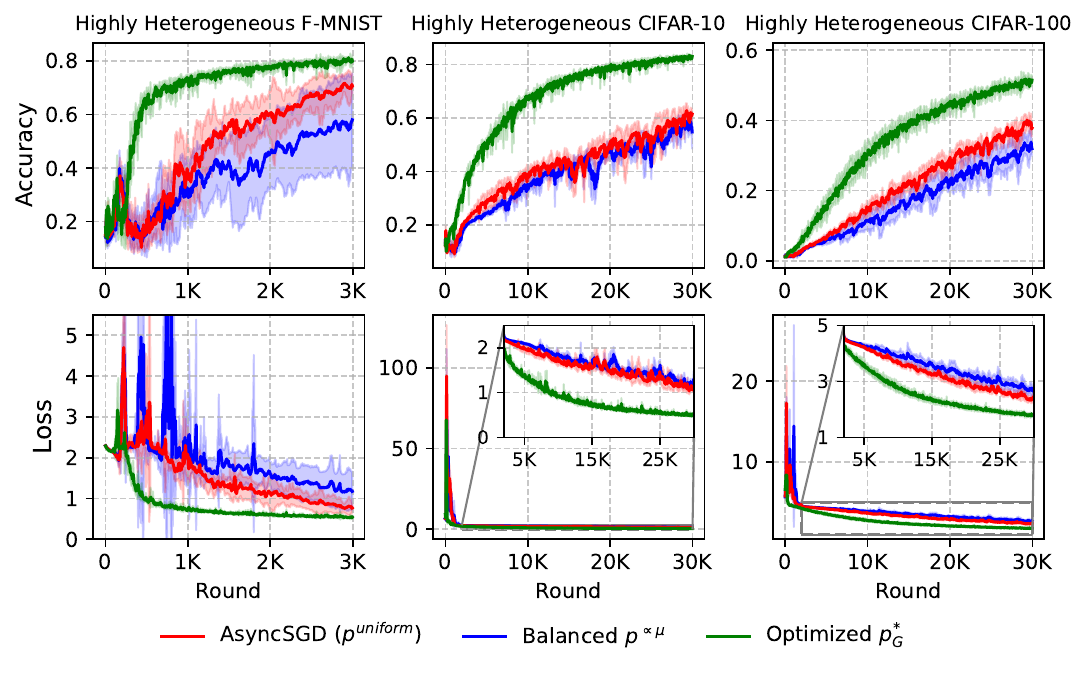}
  \caption{
   Performance on the test set at the \gls{CS} in the scenario of \Cref{num:optimize-updates}, with $n = 20$ clients and $m = 100$ tasks. Results are reported under \textbf{deterministic computation times} for three data distribution regimes: homogeneous, heterogeneous (Dirichlet), and highly heterogeneous.\\}
  \label{fig:sim_deter}
\end{figure}

\begin{figure}[htbp]
  \centering
    \includegraphics[width=0.49\textwidth]{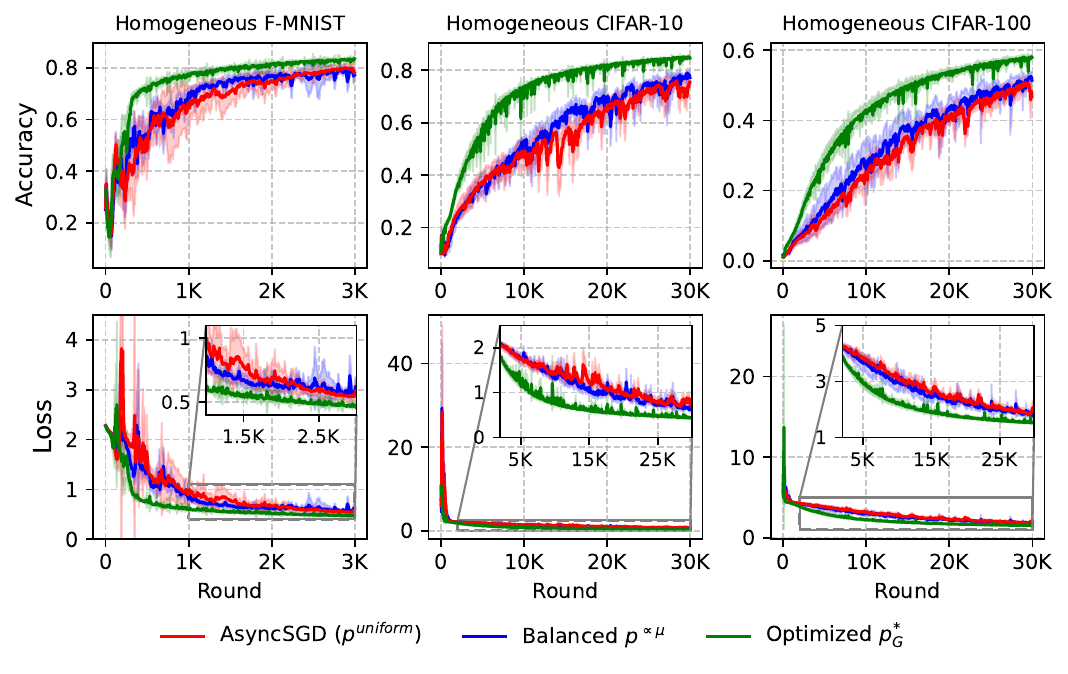}
  \hfill
    \includegraphics[width=0.49\textwidth]{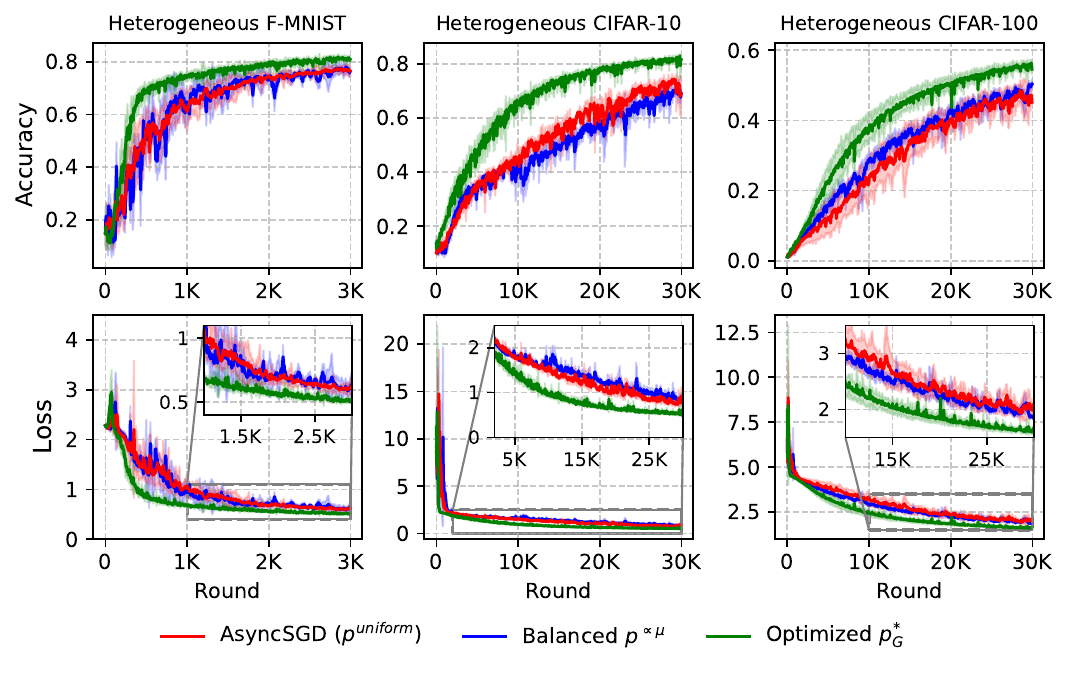}
  \\
    \includegraphics[width=0.49\textwidth]{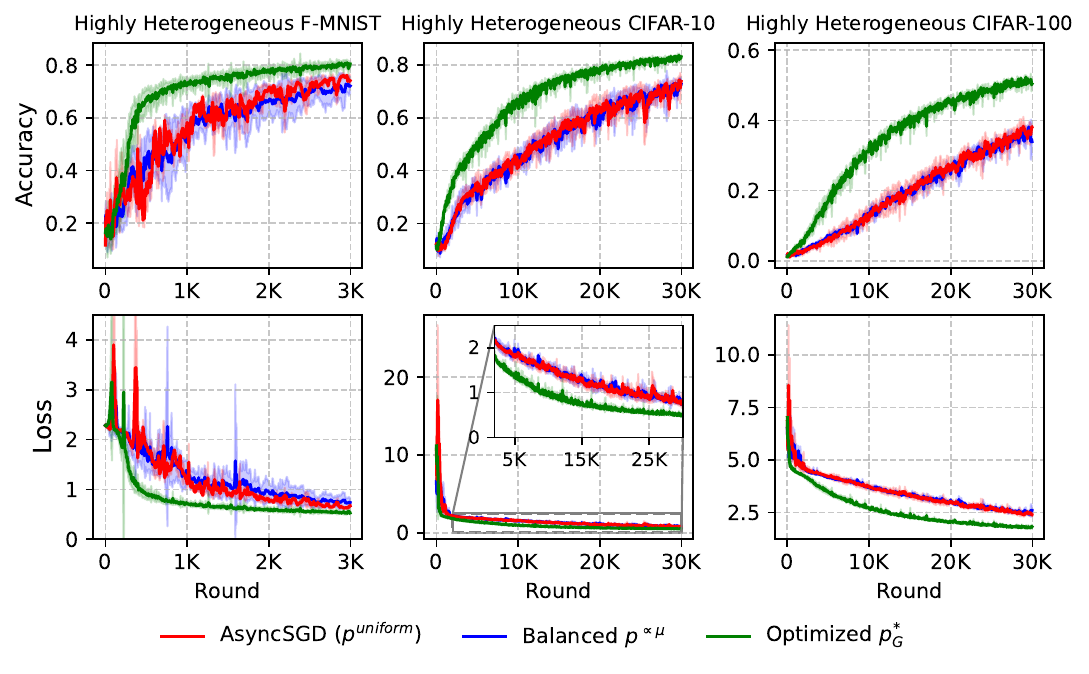}
  \caption{
   Performance on the test set at the \gls{CS} in the scenario of \Cref{num:optimize-updates}, with $n = 20$ clients and $m = 100$ tasks. Results are reported under \textbf{lognormal computation times} for three data distribution regimes: homogeneous, heterogeneous (Dirichlet), and highly heterogeneous.}
  \label{fig:sim_lognormal}
\end{figure}

\subsection{Additional Experiments for Section~\ref{num:optimize-time}} \label{app:num-additional-H}

To further assess the effectiveness of the $H$-optimized routing strategy with respect to wall-clock performance, we conduct additional experiments. Section~\ref{app:num-H-dist-heterogeneity} evaluates its robustness across different datasets (CIFAR-10 and CIFAR-100) and under more challenging heterogeneous data distributions. Section~\ref{app:num-H-computation-heterogeneity} examines its sensitivity to different computation time distributions.

\subsubsection{Robustness to Data Distribution Heterogeneity} \label{app:num-H-dist-heterogeneity}

This experiment builds on the setup of Section~\ref{num:optimize-time}, extending the evaluation to the CIFAR-10 and CIFAR-100 datasets. We assess the performance of the $H$-optimized routing strategy in terms of wall-clock time across three data distribution scenarios:

\begin{itemize}
    \item The first two match those in Section~\ref{num:optimize-time}: a \textbf{homogeneous} setting where data are i.i.d. across clients, and a \textbf{heterogeneous} setting where data are distributed according to a Dirichlet distribution ($\text{Dir}_n(0.5)$). KMNIST results for these settings are already presented in Section~\ref{num:optimize-time}.
    \item The third is a \textbf{highly heterogeneous} scenario, similar to that in Section~\ref{app:num-dist-heterogeneity}, where each client sees only a limited subset of labels. Specifically, for KMNIST and CIFAR-10, each client receives 3 unique labels out of 10; for CIFAR-100, 31 labels out of 100 are assigned per client, distributed sequentially and cyclically.
\end{itemize}

We adopt the same network dynamics as in Section~\ref{num:optimize-time} and compare the evolution of test accuracy and loss under four routing strategies: uniform, balanced, $G$-optimized, and $H$-optimized. To control for initialization effects, all models were initialized with the same weights across strategies. For each dataset and routing strategy, we ran 5 independent simulations with different random seeds and evaluated performance on an unseen, label-balanced test set.

The results in \Cref{fig:time_results_exp} support the conclusions of Section~\ref{num:optimize-time}. The $H$-optimized and balanced strategies perform roughly similarly in homogeneous settings. However, when data are heterogeneous, the $H$-optimized routing consistently outperforms the alternatives, achieving higher accuracy and demonstrating greater robustness and stability. In contrast, the balanced routing strategy becomes unstable in both heterogeneous and highly heterogeneous settings for KMNIST and CIFAR-10.

\begin{figure*}[hbtp]
\begin{center}
\centerline{\includegraphics[width=\textwidth]{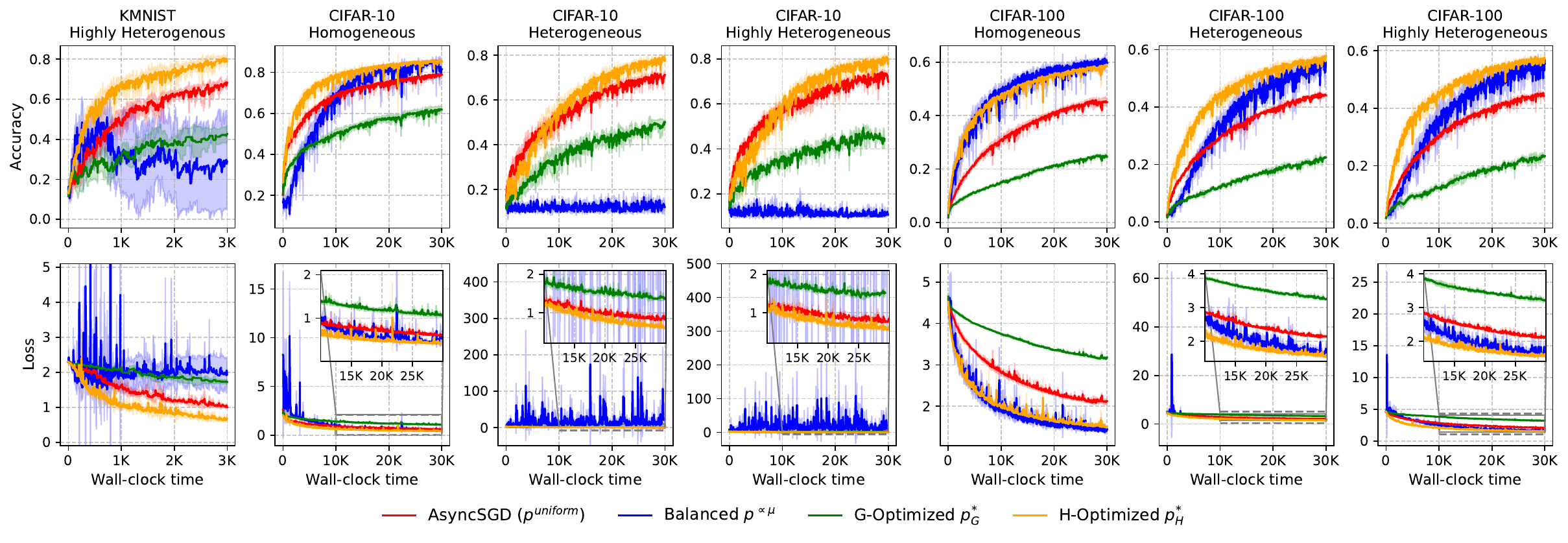}}
\caption{Performance on the test set with respect to wall-clock time at the \gls{CS} in the scenario of \Cref{num:optimize-time}, using $n = 30$ clients and $m = 30$ tasks, under homogeneous, heterogeneous, and highly heterogeneous data distributions. Solid lines denote averages over independent runs; shaded areas indicate standard deviations.}
\label{fig:time_results_exp}
\end{center}
\end{figure*}

\subsubsection{Robustness to Computation Time Distributions} \label{app:num-H-computation-heterogeneity}

To evaluate the robustness of $H$-optimized routing under varying computation time distributions, we retain the setup from \Cref{num:optimize-time}, modifying only the distribution of computation times while keeping the average task completion rate per client fixed.

As before, we study two scenarios: (i) deterministic times, where each client~$i$ has fixed duration $\frac{1}{\mu_i}$; and (ii) lognormal times, a heavy-tailed distribution modeling realistic delays. In the lognormal case, the mean is set to $\frac{1}{\mu_i}$ and the standard deviation of the underlying normal distribution is fixed at $\sigma_s = 1$ to ensure constant coefficient of variation across clients. The $\mu_i$ values match those from the exponential case in \Cref{num:optimize-time}.

We compare the $H$-optimized routing strategy $p^*_H$ (computed as in \Cref{num:optimize-time}) with uniform, balanced, and $G$-optimized strategies under homogeneous and heterogeneous data settings for KMNIST and CIFAR-100. Heterogeneity is introduced via Dirichlet sampling ($\text{Dir}_n(0.5)$) or disjoint label partitions (see Section~\ref{app:num-H-dist-heterogeneity}). KMNIST is trained for 3{,}000 wall-clock time units over 5 runs; CIFAR-100 uses standard preprocessing, is trained for 30{,}000 wall-clock time units, and repeated 3 times. Accuracy and loss are logged on a label-balanced test set. All models share the same initialization to ensure fair comparison.

The results in \Cref{fig:time_sim_deter,fig:time_sim_lognormal} show that performance is largely insensitive to the choice of computation time distribution, with trends closely aligning with those observed under the exponential assumption.
Across all configurations, $H$-optimized routing effectively balances update frequency and gradient staleness.

\begin{figure*}[hbtp]
\begin{center}
\centerline{\includegraphics[width=\textwidth]{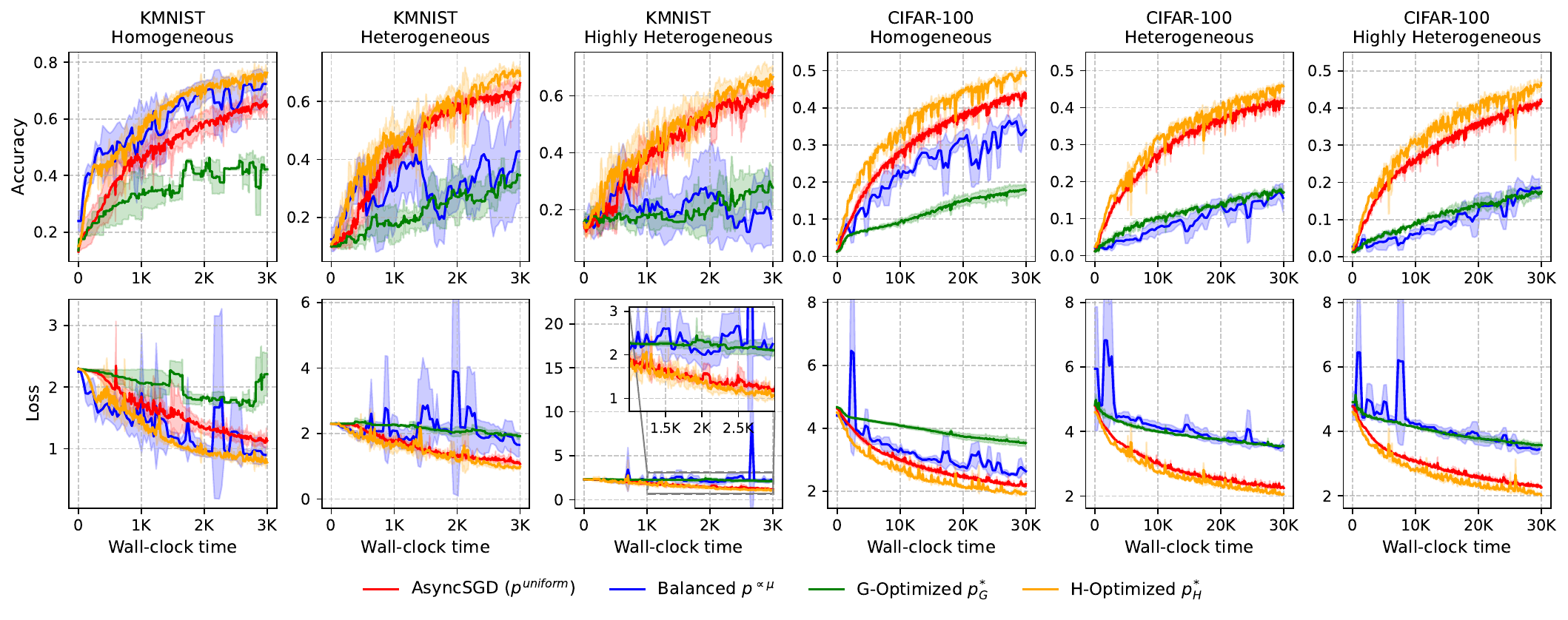}}
\caption{Performance on the test set with respect to wall-clock time at the \gls{CS} in the scenario of \Cref{num:optimize-time}, with $n = 30$ clients and $m = 30$ tasks. Results are reported for the KMNIST and CIFAR-100 datasets under \textbf{deterministic computation times}, across three data distribution regimes: homogeneous, heterogeneous, and highly heterogeneous. Solid lines denote averages over independent runs; shaded areas indicate standard deviations.}
\label{fig:time_sim_deter}
\end{center}
\end{figure*}

\begin{figure*}[hbtp]
\begin{center}
\centerline{\includegraphics[width=\textwidth]{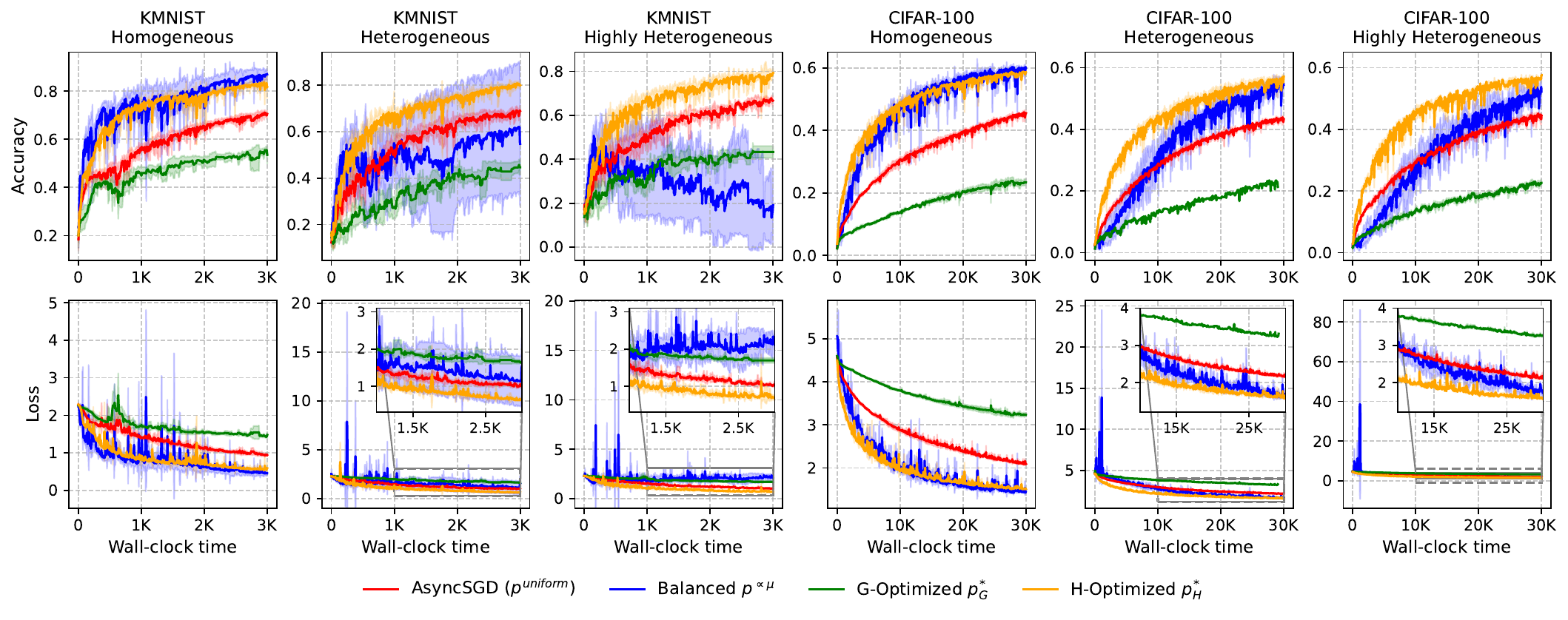}}
\caption{Performance on the test set with respect to wall-clock time at the \gls{CS} in the scenario of \Cref{num:optimize-time}, with $n = 30$ clients and $m = 30$ tasks. Results are reported for the KMNIST and CIFAR-100 datasets under \textbf{lognormal computation times}, across three data distribution regimes: homogeneous, heterogeneous, and highly heterogeneous. Solid lines denote averages over independent runs; shaded areas indicate standard deviations.}
\label{fig:time_sim_lognormal}
\end{center}
\end{figure*}

\end{document}